\documentclass{article} % For LaTeX2e
\usepackage{iclr2026_conference,times}

\usepackage{amsmath,amsfonts,bm}

\def\eqref#1{equation~\ref{#1}}
\def\1{\bm{1}}

\DeclareMathAlphabet{\mathsfit}{\encodingdefault}{\sfdefault}{m}{sl}
\SetMathAlphabet{\mathsfit}{bold}{\encodingdefault}{\sfdefault}{bx}{n}

\usepackage{url}

\newcommand{\ourtitle}{Whom to Trust? Adaptive Collaboration in Personalized Federated Learning}
\title{\ourtitle}
\author{Amr Abourayya \\
Lamarr Institute, TU Dortmund, and\\
Institute for AI in Medicine, UK Essen\\
Germany \\
\texttt{\{amr.abourayya\}@tu-dortmund.de} \\
\And
Jens Kleesiek\\
Institute for AI in Medicine,\\ UK Essen \ \ Germany \\
\And
Bharat Rao\\
Carenostics\\
USA
\And
Michael Kamp\\
Lamarr Institute, TU Dortmund, and \\
Institute for AI in Medicine, UK Essen\\ Germany
}

\input{preamble}

\SetCommentSty{textnormal}
\SetKwInOut{Output}{output}
\SetKwInOut{Input}{input}
\DontPrintSemicolon
\SetNoFillComment
\SetAlgoNoLine

\newcommand{\pfedct}{\textsc{FedMosaic}\xspace}
\newcommand{\fedmosaic}{\pfedct}

\newcommand{\fedavg}{\textsc{FedAvg}\xspace}

\theoremstyle{plain}
\newtheorem{theorem}{Theorem}
\newtheorem{lemma}[theorem]{Lemma}
\newtheorem{remark}[theorem]{Remark}
\newtheorem{proposition}{Proposition}
\newtheorem{definition}[theorem]{Definition}
\newtheorem{corollary}{Corollary}

\newtheorem{notation}[theorem]{Notation}
\newtheorem{conjecture}[theorem]{Conjecture}
\newtheorem{assumption}[theorem]{Assumption}
\newtheorem{assumptions}[theorem]{Assumptions}
\newtheorem{observation}[theorem]{Observation}
\newtheorem{fact}[theorem]{Fact}
\newtheorem{claim}[theorem]{Claim}
\newtheorem{problem}[theorem]{Problem}
\newtheorem{open}[theorem]{Open Problem}
\newtheorem{hypothesis}[theorem]{Hypothesis}
\newtheorem{question}[theorem]{Question}
\newtheorem{case}{Case}
\newtheorem*{proposition*}{Proposition}
\newtheorem*{lemma*}{Lemma}
\newtheorem*{corollary*}{Corollary}

\definecolor{BestColor}{HTML}{1B9E77}     % Teal
\definecolor{SecondColor}{HTML}{D95F02}   % Orange
\definecolor{BaselineColor}{HTML}{377EB8} % Steel Blue
\definecolor{WorstColor}{HTML}{E41A1C}    % Red

\newcommand{\best}[1]{\textcolor{BestColor}{#1}}       
\newcommand{\second}[1]{\textcolor{SecondColor}{#1}}   
\newcommand{\baseline}[1]{\textcolor{BaselineColor}{#1}} 
\newcommand{\worst}[1]{\textcolor{WorstColor}{#1}}     

\newcommand{\worstcen}[1]{\textcolor{Black}{#1}}      % Clearly worse
\newcommand{\acc}[2]{\makecell[{{c}}]{\strut $#1$ \\[-3pt] \tiny ($#2$)}}

\begin{document}

\maketitle

\begin{abstract}
Data heterogeneity poses a fundamental challenge in federated learning (FL), especially when clients differ not only in distribution but also in the reliability of their predictions across individual examples. While personalized FL (PFL) aims to address this, we observe that many PFL methods fail to outperform two necessary baselines, local training and centralized training. This suggests that meaningful personalization only emerges in a narrow regime, where global models are insufficient, but collaboration across clients still holds value. Our empirical findings point to two key ingredients for success in this regime: adaptivity in collaboration and fine-grained trust, at the level of individual examples. We show that these properties can be achieved within federated semi-supervised learning, where clients exchange predictions over a shared unlabeled dataset. This enables each client to align with public consensus when it is helpful, and disregard it when it is not, without sharing model parameters or raw data. As a concrete realization of this idea, we develop \pfedct, a personalized co-training method where clients reweight their loss and their contribution to pseudo-labels based on per-example agreement and confidence. \pfedct outperforms strong FL and PFL baselines across a range of non-IID settings, and we prove convergence under standard smoothness, bounded-variance, and drift assumptions. In contrast to many of these baselines, it also outperforms local and centralized training. These results clarify when federated personalization can be effective, and how fine-grained, trust-aware collaboration enables it.

\end{abstract}

\section{Introduction}
\label{sec:introduction}
Federated learning (FL) enables collaborative machine learning across distributed data sources without direct data sharing. Classical methods such as FedAvg \citep{mcmahan2017communication}, aim to train a single global model across all clients. This approach can succeed when data distributions are sufficiently similar, but collapses under strong distributional shifts. In highly heterogeneous settings, the promise of collaboration breaks down: models trained jointly may perform worse than models trained independently.

Personalized Federated Learning (PFL) addresses this challenge by shifting the goal. Rather than optimizing a shared global model, the goal is to use collaboration to improve each client's personalized model. For example, Tab.~\ref{tab:hybrid-compact} shows that in heterogeneous regimes both FL and even centralized training perform worse than local training,i.e., clients learning independently without any communication. This underlines the requirement for PFL, but also highlights an often-overlooked baseline: when no method outperforms local training, collaboration is not just ineffective—it is detrimental. Yet many PFL methods fail to beat this baseline (cf. Tab.~\ref{tab:hybrid-compact}), casting doubt on their utility.

\clearpage
\begin{wraptable}{r}{0.44\textwidth}
\scriptsize
%\vspace{100mm}
\centering
\setlength{\tabcolsep}{2.5pt}
\renewcommand{\arraystretch}{1.}
\caption{\textbf{Average test Accuracy on DomainNet and Office-10 dataset} (details in sec.\ref{sec:experiments}). Most personalized FL methods fail to surpass local training baseline. \pfedct exceeds both core baselines through adaptive, example-level collaboration. Color Map: \baseline{baselines}, \worst{worse than baselines}, \second{worse than local training}, \best{better than baselines.}}
\begin{tabular}{@{}l l l l@{}}
\toprule
& \textbf{Method} & \textbf{DomainNet} & \textbf{Office} \\
\midrule
& Centralized  & \baseline{$66.24$ \tiny $(0.4)$} & \baseline{$40.92$ \tiny $(0.6)$} \\
\midrule
\multirow{2}{*}{\rotatebox[origin=c]{90}{\textbf{FL}}}
  & FedAvg     & \worst{$31.00$ \tiny $(0.8)$} & \worst{$37.25$ \tiny $(0.8)$} \\
  & FedProx    & \worst{$55.23$ \tiny $(0.1)$} & \second{$58.39$ \tiny $(0.3)$} \\
\midrule
\multirow{4}{*}{\rotatebox[origin=c]{90}{\textbf{PFL}}}
  & Per-FedAvg & \second{$72.48$ \tiny $(0.4)$} & \second{$71.92$ \tiny $(0.5)$} \\
  & pFedMe     & \second{$75.21$ \tiny $(0.5)$} & \second{$74.83$ \tiny $(0.7)$} \\
  & APFL       & \second{$80.59$ \tiny $(0.3)$} & \second{$80.91$ \tiny $(0.1)$} \\
  & FedPHP     & \second{$78.25$ \tiny $(0.6)$} & \second{$76.36$ \tiny $(0.4)$} \\
\midrule
& Local Training & \baseline{$84.64$ \tiny $(0.1)$} & \baseline{$86.79$ \tiny $(0.4)$} \\
\midrule
& \textbf{\pfedct} & \best{$\mathbf{87.44}$ \tiny $(0.02)$} & \best{$\mathbf{89.06}$ \tiny $(0.01)$} \\
\bottomrule
\end{tabular}
\vspace{-2mm}
\label{tab:hybrid-compact}
\vspace{-3mm}
\end{wraptable}

This widespread failure to measure true collaborative gain arises because "personalization" is often treated as a vague remedy for heterogeneity without a clear underlying principle. We argue that progress requires a new foundation. Personalization shouldn't be a default modification to an existing FL algorithm; it should emerge from a principled understanding of what each client needs and how collaboration can help. A meaningful PFL solution must adapt the degree and nature of collaboration based on client context. It must also account for heterogeneity not just between clients, but at the level of individual examples. Clients may align on some concepts (e.g., identifying cats) and diverge on others (e.g., identifying specific dog breeds), and collaboration should reflect this granularity.

In formal terms, PFL aims to minimize the sum of local risks across $m$ clients with heterogeneous data distribution $\mathcal{D}_{i}$ and personalized models $h_1, \dots, h_m$:
\[
\min _{h_1, \ldots, h_m} \sum_{i=1}^m \mathbb{E}_{(x, y) \sim \mathcal{D}_i}\left[\mathcal{L}\left(h_i(x), y\right)\right] \enspace .
\]
In this setting, local model may outperform global or centralized models, making strong local and centralized baselines éssential. The key trade-off between the massive data access of a centralized model versus the specialization of a local one, is the central tension PFL must navigate in an adaptive and data-specific way.

While federated learning can adapt by weighing parameters according to similarity~\citep{huang2021personalized, zhangpersonalized}, data-specific collaborations require a shift in mechanism. Rather than aggregating model parameters, we propose to use federated semi-supervised learning~\citep{bistritz2020distributed,abourayya2025little} where clients share predictions on a public dataset. Collaboration is achieved by enforcing consensus between clients. We propose to adapt this consensus mechanism so that clients can contribute only on examples where they have expertise and can selectively trust others based on their demonstrated competence. Two clients familiar with cats can confidently collaborate on a new cat photo, while a client that has only seen cars should not influence the labeling of cat images. This form of selective, example-level trust is fundamentally difficult to achieve through parameter averaging alone.

In this work, we demonstrate this principle in practice. We propose a personalized Federated Co-Training approach (\pfedct) that enables adaptive, fine-grained collaboration through two key mechanisms: a dynamic weighting strategy allowing clients to balance global and local signals in each communication round, and an expertise-aware consensus mechanism that weights peer contributions by their competence on different data regions. Both mechanisms operate on predictions over a public dataset, enabling personalization that is responsive to the data's true structure.

While \pfedct achieves state-of-the-art empirical performance across benchmarks, its main contribution is conceptual. It redefines personalization as a question of collaborative structure, not just algorithm design. Our results show that principled, example-level collaboration can unlock the full potential of personalized federated learning.

\section{Related Work}
\label{sec:related_work}
Federated Learning (FL) aims to train models collaboratively across decentralized clients without compromising data privacy. However, heterogeneous data distributions across clients (non-IID settings) present a persistent challenge that degrades performance. Approaches addressing heterogeneity broadly fall into two categories: traditional FL and personalized FL (PFL) methods. We review these groups in relation to our method, FEDMOSAIC.

\paragraph{Traditional Federated Learning:}

Traditional federated learning methods typically learn a single global model. \fedavg ~\citep{mcmahan2017communication} averages local models but struggles under non-IID data due to client drift. Subsequent methods attempt to correct this: SCAFFOLD \citep{karimireddy2020scaffold} uses control variates to correct the local updates, FedProx~\citep{li2020federated} adds a proximal term to each client's loss function to stabilize training, and FedDyn~\citep{acar2021federated} introduces dynamic regularization. Others use representation alignment, such as MOON~\citep{li2021model}, which applies a contrastive loss to align local and global features. These methods implicitly assume a global model can suffice, which may fail under strong heterogeneity. Moreover, parameter sharing can pose privacy risks~\citep{zhu2019deep, abourayya2025little}.

\paragraph{Personalized Federated learning (PFL):}

Personalized Federated learning methods tailor models to individual clients, addressing non-IID challenges through different strategies.

\textbf{Meta-learning and Regularization-Based Methods} optimize a shared initialization or constrain local updates. E.g., Per-FedAvg~\citep{fallah2020personalized} learns a shared initialization, while Ditto~\citep{li2021ditto} regularizes local updates toward a global reference. PFedMe~\citep{t2020personalized} applies bi-level optimization to decouple personalization from global learning.
\textbf{Personalized Aggregation strategies} dynamically aggregate models based on client similarity or adaptive weighting. APFL~\citep{deng2020adaptive} introduces an adaptive mixture of global and local models, allowing clients to interpolate between shared and personalized parameters based on their data distribution. FedAMP~\citep{huang2021personalized} uses attention to weight client contributions based on similarity. Other methods select collaborators (e.g., FedFomo~\citep{zhangpersonalized}, FedPHP~\citep{li2021fedphp}) or apply layer-wise attention (FedALA~\citep{zhang2023fedala}).
\textbf{Model Splitting Architectures} partition models into shared and personalized components. FedPer~\citep{arivazhagan2019federated}eeps shared base layers and personalizes top layers. FedRep~shares a backbone but personalizes the head.\citep{collins2021exploiting} shares a backbone but personalizes the head. FedBN~\citep{li2021fedbn} personalizes batch normalization layers to tackle feature shift. 
Other recent methods such as FedAS~\citep{yang2024fedas}, GPFL~\citep{zhang2023gpfl}, and FedBABU~\citep{oh2021fedbabu} disentangle or freeze specific parts of the model to balance generalization and personalization. PFedHN~\citep{shamsian2021personalized} uses a hypernetwork that generates personalized model parameters conditioned on client identity.
\textbf{Knowledge Distillation Approaches} transfer knowledge from global or peer models to personalized local models. FedProto~\citep{tan2022fedproto} aligns class-wise feature prototypes across clients, FedPAC~\citep{xu2023personalized} uses contrastive learning to distill knowledge into personalized models, and FedKD~\citep{wu2022communication} reduces communication cost by distilling knowledge from a teacher ensemble to lightweight client models. FedMatch~\citep{chen2021fedmatch} uses consistency regularization to unlabeled and noisy data, FedDF~\citep{lin2020ensemble} aggregates predictions via ensemble distillation, and FedNoisy~\citep{liang2023fednoisy} focuses on robust aggregation in the presence of noisy labels or adversarial participants. PerFed-CKT~\citep{cho2021personalized} enhances personalization by clustering clients with similar data distributions and facilitating knowledge transfer through logits instead of model parameters.~\cite{jeong2023personalized} proposes a fully decentralized PFL framework where clients share distilled knowledge with neighboring clients, enabling personalization without a central server.FedD2S~\citep{atapour2024fedd2s} introduces a data-free federated knowledge distillation approach that employs a deep-to-shallow layer-dropping mechanism.

Despite this progress, existing PFL methods often share several limitations:
(i) \emph{Static collaboration:} Most PFL methods rely on fixed rules (e.g., aggregation weights or model splits), lacking adaptivity to client-specific or example-level variation.
(ii) \emph{Privacy risks:} Sharing model parameters, gradients, or even soft labels may expose sensitive information.
(iii) \emph{Limited generality:} Many methods are tailored to specific heterogeneity types (e.g., label skew in case of FedMix, or feature shift in case of FedBN).
(iv) \emph{Communication / computational overhead:} Some require complex multi-model training or costly synchronization.
To overcome these limitations, we argue that PFL methods should use some form of dynamic modulation and per-example trust weighting.

\section{Personalized Federated Co-Training: Adaptive and Expert-Aware Collaboration}

We now introduce Personalized Federated Co-Training (\pfedct), a concrete realization of the principle that effective personalization arises from adaptive, data-specific collaboration. Our method builds upon the framework of federated co-training \citep{abourayya2025little}, a privacy-preserving paradigm where clients collaborate by sharing hard predictions on a shared, unlabeled public dataset, $U$  (we analyze the impact of this dataset's size and distribution in sec.\ref{sec:experiments}). This process creates a consensus pseudo-labeled dataset, which clients use to augment their local training.

While this approach avoids sharing sensitive model parameters and soft labels, it introduces two critical challenges for personalization:
\begin{enumerate}
    \item \textbf{When to trust the global signal?} A client's local data may conflict with the global consensus. Blindly trusting pseudo-labels can harm a model that is already well-specialized.
    \item \textbf{Whose predictions to trust?}  Clients possess varying levels of expertise across the data space. A naive consensus that treats all clients equally will be corrupted by noisy or misaligned predictions.
\end{enumerate}

\pfedct addresses these challenges directly with two core mechanisms: (1) dynamic loss weighting, which allows each client to adaptively decide when to trust the global signal, and (2) confidence-based aggregation, which intelligently decides whose predictions to trust.

\textbf{Dynamic Loss Weighting: Deciding When to Trust}:
To allow clients to autonomously balance global collaboration with local specialization, we introduce a dynamic weight $\lambda_i^t$, into the local objective. At each round $t$, client $i$ minimizes the combined loss:
\[
\mathcal{L}_i^t(h)=\mathcal{L}\left(h, D_i\right)+\lambda_i^t \cdot \mathcal{L}\left(h, P_t\right)
\]

where $\mathcal{D}_i$ is the client's private data and $P_t$ is the pseudo-labeled public dataset.The weight $\lambda_i^t$ modulates the influence of the global signal.Our choice of the function for computing $\lambda_i^t$was driven by the need for a smooth, bounded, and interpretable mechanism. We define it as:
\[
\lambda_i^t=\exp \left(-\frac{\mathcal{L}\left(h_{t-1}^i, P_t\right)-\mathcal{L}\left(h_{t-1}^i, D_i\right)}{\mathcal{L}\left(h_{t-1}^i, D_i\right)}\right)
\]

This exponential form satisfies several desirable properties. It ensures positivity $\left(\lambda_i^t>0\right),$ avoids discontinuities, and smoothly adjusts the client's trust based on the relative performance of its model on global versus local data. The behavior is highly intuitive:
\begin{itemize}
    \item Conflict $\left(\mathcal{L}_{\text {global }} \gg \mathcal{L}_{\text {local }}\right)$: If the consensus pseudo-labels are harmful, the global loss term increases, causing $\lambda_i^t \rightarrow 0$ and prompting the client to rely on its local data.
    \item Alignment $\left(\mathcal{L}_{\text {global }} \approx \mathcal{L}_{\text {local }}\right)$ : If the consensus is helpful and aligns with local data, $\lambda_i^t \approx 1$ achieving a balance between personalization and collaboration.
    \item Enhancement $\left(\mathcal{L}_{\text {global }}<\mathcal{L}_{\text {local }}\right)$: If the consensus provides a cleaner signal than the noisy local data, $\lambda_i^t >1$, encouraging the client to trust the collaborative signal more heavily.
\end{itemize}

\textbf{Confidence-Based Aggregation: Deciding Whose to Trust}: To address the varying expertise of clients, we replace the standard uniform aggregation of predictions with a confidence-based consensus. Instead of just sharing hard labels, each client $i$ also communicates a confidence vector $E_t^i\in (0,\infty)^{|U|}$, where $E_t^i[j]$ quantifies its estimated expertise on its prediction for example $x_j\in U$.The server then computes a weighted score matrix $S_t$ by aggregating the one-hot predictions $L_t^i$ from each client, weighted by their corresponding expertise:
\[
S_t=\sum_{i=1}^m \operatorname{diag}\left(E_t^i\right) \cdot L_t^i \in \mathbb{R}^{|U| \times C}
\]

The final consensus pseudo-label for each example is determined by the highest aggregate score:
\[
L_t[j]=\arg \max _{c \in[C]} S_t[j, c], \quad \forall j \in\{1, \ldots,|U|\}
\]

This mechanism allows clients who are more confident or reliable about specific data regions to have a greater influence on the consensus, effectively reducing the impact of noise from non-expert clients. We explore two practical instantiations for the confidence scores $E_t^i$: a class-frequency-based heuristic and an uncertainty-based score derived from the model's predictive entropy. The full procedure is detailed in Algorithm ~\ref{alg:fedmosaic}.

\SetKwFor{local}{Locally}{do}{}
\SetKwFor{coord}{At server}{do}{}
\begin{algorithm}[ht]
    \caption{Federated Co-Training with Adaptivity and Specialization (\fedmosaic)}
    \label{alg:fedmosaic}
    \KwIn{communication period $b$, $m$ clients with local datasets $D^1,\dots,D^m$ and learning algorithms $\mathcal{A}^1,\dots,\mathcal{A}^m$, unlabeled public dataset $U$, total rounds $T$}
    \KwOut{final models $h^1_T,\dots,h^m_T$}
    \vspace{0.1cm}
    initialize local models $h_0^1,\dots,h_0^m$, \quad $P \leftarrow \emptyset$\\

    \local{at client $i$ at time $t$}{
        compute local loss $\ell_{\mathrm{priv}} = \mathcal{L}(h^i_{t-1}, D^i)$\\
        compute pseudo-label loss $\ell_{\mathrm{pseudo}} = \mathcal{L}(h^i_{t-1}, P)$\\
        \colorbox{verdigreen!20}{compute adaptive weight $\lambda_t^i = \exp\left(-\frac{\ell_{\mathrm{pseudo}} - \ell_{\mathrm{priv}}}{\ell_{\mathrm{priv}}}\right)$}\\
        \colorbox{verdigreen!20}{compute loss $\ell=\ell_{\mathrm{priv}}+\lambda_t^i\ell_{\mathrm{pseudo}}$}\\
        update $h_t^i \leftarrow \mathcal{A}^i(\ell, h^i_{t-1})$\\

        \If{$t\ \%\ b = b-1$}{
             construct prediction matrix $L_t^i \in \{0,1\}^{|U| \times C}$\\
             \colorbox{verdigreen!20}{construct expertise vector $E_t^i \in (0, \infty)^{|U|}$}\\
             \colorbox{verdigreen!20}{send $(L_t^i, E_t^i)$ to server and receive $L_t$}\\
             $P \leftarrow (U, L_t)$
        }
    }

    \coord{at time $t$}{
        \colorbox{verdigreen!20}{receive $(L_t^1, E_t^1),\dots,(L_t^m, E_t^m)$ from clients}\\
        \colorbox{verdigreen!20}{compute weighted score matrix $S_t = \sum_{i=1}^m \operatorname{diag}(E_t^i) \cdot L_t^i$}\\
        \colorbox{verdigreen!20}{set pseudo-labels $L_t[j] = \arg\max_{c \in [C]} S_t[j,c] \quad \text{for all } j \in \{1, \dots, |U|\}$}\\
        send $L_t$ to all clients
    }
\end{algorithm}

\paragraph{Communication.}
In each communication round (every $b$ local steps), client $i$ sends a one-hot matrix $L_t^i \in \{0,1\}^{|U|\times C}$ and expertise vector $E_t^i \in \mathbb{R}^{|U|}$; thus it adds exactly one scalar per public example compared to federated co-training~\citep{abourayya2025little}. Encoding $L_t^i$ by class \emph{indices} (majority vote depends only on $\arg\max$) uses $\lceil \log_2 C\rceil$ bits per example instead of $C$ bits, and quantizing expertise to $b_E$ bits gives a per-round uplink budget
$
B_{\text{\pfedct}} \;=\; |U|\,\bigl(\lceil \log_2 C\rceil + b_E\bigr)\;\text{bits}.$
By contrast, parameter sharing (e.g., \fedavg) uploads $32P$ bits for a model with $P$ parameters. For example, as in our Fashion-MNIST experiments with $|U|=10^4$ and $C=10$, choosing $b_E=8$ gives $B_{\text{\pfedct}}=10^4(4+8)=1.2\!\times\!10^5$ bits ($\approx 15$ KB) per client and round; parameter sharing instead communicates $\approx 2.6MB$, so \fedmosaic reduces communication by a factor of $\approx 177$.

\paragraph{Convergence under dynamic pseudo-labels:}
To provide theoretical support, we analyze the convergence behavior of \pfedct under standard assumptions in stochastic optimization. Our goal is to characterize the rate at which each client's objective approaches a stationary point, despite the dynamic pseudo-labeling and the heterogeneity of local objectives.

We assume standard conditions, including smoothness of the loss functions, bounded gradient variance, and bounded drift of pseudo-labels across rounds. These assumptions reflect the structure of \pfedct, where local objectives are updated periodically but converge due to the stabilization of pseudo-labels as shown by \citet{abourayya2025little}. 
\begin{assumptions}
The following conditions hold for each client $ i \in [m] $ at round $ t $:
\begin{enumerate}
    \item Each loss function $\mathcal{L}_i^{\mathrm{local}}$ and $\mathcal{L}_i^{\mathrm{global}, t}$ is $L(1+e)^{-1}$-smooth.
    \item The gradient estimator $g_i^t$ is unbiased and has bounded variance:
    \[
    \mathbb{E}[g_i^t] = \nabla \mathcal{L}_i^t(\theta_t), \quad \mathbb{E}[\|g_i^t - \nabla \mathcal{L}_i^t(\theta_t)\|^2] \leq \sigma^2.
    \]    
    \item The global loss has bounded gradients: $\|\nabla \mathcal{L}_i^{\mathrm{global}, t}(\theta)\| \leq G$ for all $\theta$ and $t$.
    \item The objective drift is bounded:
    $
    |\mathcal{L}_i^{t+1}(\theta) - \mathcal{L}_i^t(\theta)| \leq \delta, \quad \forall \theta.
    $
    \item The per-sample gradient variance is bounded:
    \begin{equation*}
    \begin{split}
    \hspace{-0.9cm}\mathbb{E}_{x\in D_i}\left[\left\|\nabla_\theta\ell(\theta,x,\hat{y}^t)-\nabla\mathcal{L}_i^{local,t}\right\|^2\right]\leq \overline\sigma^2\enspace,\enspace
    \mathbb{E}_{x\in U}\left[\left\|\nabla_\theta\ell(\theta,x,\hat{y}^t)-\nabla\mathcal{L}_i^{global,t}\right\|^2\right]\leq \tilde\sigma^2\\
    \end{split}
    \end{equation*}
\end{enumerate}
\end{assumptions}
Under these conditions, we establish that \pfedct converges to an approximate stationary point. Specifically, after $T$ communication rounds, the average squared gradient norm decreases at a rate of $\mathcal{O}(1/T)$ plus additive terms accounting for local and global variance and pseudo-label drift.
\setcounter{proposition}{0}
\begin{proposition}[Convergence of \pfedct]
\label{theorem-convergence}
Let each client's objective at round $t$ be
\[
\mathcal{L}_i^t(\theta) = \mathcal{L}_i^{\mathrm{local}}(\theta) + \lambda_i^t \mathcal{L}_i^{\mathrm{global}, t}(\theta),\text{ where }\lambda_i^t = \exp\left(-\frac{\mathcal{L}_i^{\mathrm{global}}(\theta_t) - \mathcal{L}_i^{\mathrm{local}, t}(\theta_t)}{\mathcal{L}_i^{\mathrm{local}, t}(\theta_t)}\right),
\]
and $\mathcal{L}_i^{\mathrm{global}, t}$ may change at each round due to pseudo-label updates. Under Assumptions 1-5, for a fixed step size $0<\eta\leq(2L)^{-1}$ and $\min_i |D_i|=d$, after $T$ rounds of \pfedct, it holds that
\[
\frac{1}{T} \sum_{t=0}^{T-1} \mathbb{E}[\|\nabla \mathcal{L}_i^t(\theta_t)\|^2]
\leq \frac{4L\left(\mathcal{L}_i^0 - \mathcal{L}_i^*\right)}{T} + \frac{\overline\sigma^2}{2Ld} +  \frac{e^2\tilde\sigma^2}{2L|U|} + 2\delta\enspace .
\]
\end{proposition}
The proof is provided in Appendix~\ref{app-proof-theorm}.
\citet{abourayya2025little} show that under the assumption of increasing local accuracy, pseudo-labels stabilize after some round $t_0$, so the assumption of a bounded change in the client objective is realistic. In fact, the global loss term effectively becomes stationary under these assumptions quickly and the expected drift becomes negligibly small as $t$ increases.

\paragraph{Client-Level Privacy:}
In each round $t$, client $i$ communicates a hard-label matrix $L_t^i \in \{0,1\}^{|U|\times C}$ (one-hot predictions on $U$) and an \emph{expertise} vector $E_t^i \in \mathbb{R}^{|U|}$ (one scalar per $u\in U$). Compared to \citet{abourayya2025little}, which releases only $L_t^i$, the present protocol adds exactly one real value per unlabeled example.
We apply the XOR mechanism to $L_t^i$. For this, \citet{abourayya2025little} showed that for on-average replace-one stable learning algorithms the sensitivity $s^\star$ of $L_t^i$ is bounded, yielding a per-round $\varepsilon_L$-DP guarantee at the client level.
For the expertise scores $E_t^i$ we apply the Gaussian mechanism~\citep{dwork2014algorithmic} with variance $\sigma^2$. Since the expertise scores are in $[0,1]$ for class frequencies and in $[0,\log C]$ for predictive entropy, the (per-coordinate) sensitivity of $E_t^i$ is bounded, which yields $(\varepsilon_E,\delta)$-DP with
\[
\varepsilon_E \;=\; \frac{c\sqrt{|U|}}{\sigma}\,\sqrt{2\ln\!\bigl(1.25/\delta\bigr)}\enspace,
\]
where $c=1$ for class frequencies and $c=\log C$ for predictive entropy.
Combined, these two mechanisms on $L_t^i$ and $E_t^i$ yield $(\varepsilon_L+\varepsilon_E,\delta)$-DP for \pfedct{} in each round.

\begin{table}[t]
    \centering
    \caption{Average test accuracy (\%) under pathological and practical Non-IID Settings for $m=15$ clients. 
    Color Map: \baseline{baselines}, \worst{worse than both baselines}, 
    \second{worse than local training}, \best{better than both baselines.}}
    \resizebox{0.85\textwidth}{!}{
    \begin{tabular}{l|l|c|c|c|c}
        \hline
        & \multirow{2}{*}{\textbf{Method}} 
        & \multicolumn{2}{c|}{\textbf{Pathological non-IID}} 
        & \multicolumn{2}{c}{\textbf{Practical non-IID}} \\
        \cline{3-6}
        & & \textbf{Fashion-MNIST} & \textbf{CIFAR-10} & \textbf{Fashion-MNIST} & \textbf{CIFAR-10} \\
        \hline
        & Centralized   & \baseline{$99.28$ \tiny $(0.1)$} & \baseline{$87.90$ \tiny $(0.1)$} & \baseline{$99.28$ \tiny $(0.03)$} & \baseline{$87.90$ \tiny $(0.04)$} \\
        & Local training& \baseline{$99.32$ \tiny $(0.02)$} & \baseline{$88.01$ \tiny $(0.01)$} & \baseline{$98.23$ \tiny $(0.01)$} & \baseline{$83.91$ \tiny $(0.2)$} \\
        \hline
        \multirow{4}{*}{\rotatebox[origin=c]{90}{\textbf{FL}}}
          & FedAvg  & \worst{$76.72$ \tiny $(0.1)$} & \worst{$64.42$ \tiny $(0.2)$} & \worst{$83.71$ \tiny $(0.2)$} & \worst{$70.28$ \tiny $(0.4)$} \\
          & FedProx & \worst{$77.88$ \tiny $(0.3)$} & \worst{$70.25$ \tiny $(0.2)$} & \worst{$84.14$ \tiny $(0.3)$} & \worst{$73.35$ \tiny $(0.4)$} \\
          & FedCT   & \worst{$78.15$ \tiny $(0.01)$} & \worst{$73.91$ \tiny $(0.02)$} & \worst{$85.27$ \tiny $(0.01)$} & \worst{$74.39$ \tiny $(0.01)$} \\
          & FedBN   & \worst{$78.04$ \tiny $(0.3)$} & \worst{$81.35$ \tiny $(0.5)$} & \worst{$85.39$ \tiny $(0.3)$} & \worst{$80.41$ \tiny $(0.7)$} \\
        \hline
        \multirow{6}{*}{\rotatebox[origin=c]{90}{\textbf{PFL}}}
          & Per-FedAvg & \worst{$98.63$ \tiny $(0.02)$} & \worst{$87.20$ \tiny $(0.01)$} & \worst{$97.11$ \tiny $(0.01)$} & \worst{$81.37$ \tiny $(0.2)$} \\
          & Ditto      & \best{$99.37$ \tiny $(0.01)$} & \second{$87.94$ \tiny $(0.01)$} & \worstcen{$98.39$ \tiny $(0.02)$} & \worst{$83.89$ \tiny $(0.04)$} \\
          & pFedMe     & \worst{$74.80$ \tiny $(0.4)$} & \worst{$81.47$ \tiny $(0.3)$} & \worst{$80.01$ \tiny $(0.1)$} & \worst{$81.61$ \tiny $(0.4)$} \\
          & APFL       & \worst{$99.26$ \tiny $(0.04)$} & \second{$87.98$ \tiny $(0.01)$} & \worst{$97.96$ \tiny $(0.03)$} & \worst{$83.81$ \tiny $(0.2)$} \\
          & FedPHP     & \second{$99.30$ \tiny $(0.01)$} & \worst{$87.90$ \tiny $(0.01)$} & \worstcen{$98.40$ \tiny $(0.01)$} & \worst{$83.75$ \tiny $(0.03)$} \\
          & PerFed-CKT & \best{$99.34$ \tiny $(0.01)$} & \second{$87.95$ \tiny $(0.01)$} & \worst{$98.20$ \tiny $(0.01)$} & \worst{$83.87$ \tiny $(0.03)$} \\
        \hline
        & \textbf{\pfedct} & \best{$\mathbf{99.40}$ \tiny $(0.01)$} & \best{$\mathbf{88.03}$ \tiny $(0.01)$} & \worstcen{$\mathbf{98.43}$ \tiny $(0.01)$} & \worstcen{$\mathbf{86.15}$ \tiny $(0.01)$} \\
        \hline
    \end{tabular}
    }
\label{tab:label-skew}
\end{table}

\section{Empirical Evaluation}
\label{sec:experiments}

In this section, we evaluate \pfedct against a suite of strong baselines in three challenging heterogeneity scenarios: (1) label skew, (2) feature shift, and (3) a hybrid setting combining both. We evaluate our method against FL (FedAvg, FedProx, FedCT, FedBN), state-of-the-art PFL methods (Per-FedAvg, Ditto, pFedMe, APFL, FedPHP, PerFed-CKT), and crucial local training and centralized baselines, which are essential for measuring true collaborative benefit. Centralized training refers to applying the local training algorithm on the pooled data from all clients, as if it were stored in a single location. Local training refers to each client training a model independently using only its own local data, without any collaboration.

\paragraph{Experimental Setup:}
A core component of our method is the shared, unlabeled public dataset $U$. Following standard practice in semi-supervised learning, for each experiment this dataset is a small, class-balanced sample from the original training set, omitting its labels. This ensures that $U$ is drawn IID from the global training distribution and is disjoint from every client dataset $D_i$ ($U \cap D_i = \varnothing$); since the $D_i$ are non-IID, $U$'s distribution differs from each $D_i$. This way, $U$ provides a comprehensive view of the label space, even when clients' private data is highly skewed.

We set the size of \(U\) to: CIFAR-10—\(3{,}000\) samples; Fashion-MNIST—\(2{,}250\) samples; DomainNet—\(300\) samples; and Office-10—\(80\) samples. A comprehensive ablation study detailing the impact of the public dataset's size and distribution as well as an investigation of individual clients' losses, is provided in the Appendix \ref{app:additional-exp}.

\paragraph{Label Skew:}
We first evaluate \pfedct under label distribution skew, a common protocol where clients see only subsets of the available classes. We test on two variants: a "pathological" setting where each of the $15$ clients on Fashion-MNIST and CIFAR-10 holds data from only $2$ classes, and a more practical setting where label proportions are drawn from a Dirichlet distribution. These settings are widely adopted in the literature~\citep{t2020personalized,fallah2020personalized,zhang2023eliminating,zhang2023fedcp,zhang2023gpfl}. For these experiments, we use the class-frequency-based confidence score, a natural fit for scenarios dominated by class imbalance.

As shown in Table \ref{tab:label-skew}, \pfedct achieves top performance across all settings. In the pathological case on CIFAR-10, it scores $0.8803$, surpassing all PFL methods and, crucially, the strong local training baseline ($0.8801$). This result is significant: it demonstrates that \pfedct’s adaptive collaboration successfully extracts useful signals from peers without being corrupted by their extreme data skew, achieving a better outcome than local training. Performance trends are similar in the practical scenario, confirming the method's robustness to varying degrees of label imbalance.

\begin{table}[t]
\setlength{\tabcolsep}{2.5pt}
\centering
\caption{Average test accuracy (\%) on the Office-10 and DomainNet datasets in feature shift scenarios. 
For Office-10: A, C, D, W = Amazon, Caltech, DSLR, WebCam. 
For DomainNet: C, I, P, Q, R, S = Clipart, Infograph, Painting, Quickdraw, Real, Sketch. 
Color Map: see Table~\ref{tab:label-skew}.}
\resizebox{0.8\textwidth}{!}{
\begin{tabular}{l|l|cccc|cccccc}
\hline
& \textbf{Method} & \multicolumn{4}{c|}{\textbf{Office-10}} & \multicolumn{6}{c}{\textbf{DomainNet}} \\
\cline{3-12}
& & \textbf{A} & \textbf{C} & \textbf{D} & \textbf{W} & \textbf{C} & \textbf{I} & \textbf{P} & \textbf{Q} & \textbf{R} & \textbf{S} \\
\hline
& Centralized &
\baseline{\acc{74.03}{0.1}} & \baseline{\acc{58.24}{0.2}} & \baseline{\acc{79.12}{0.2}} & \baseline{\acc{78.52}{0.01}} &
\baseline{\acc{70.53}{0.4}} & \baseline{\acc{30.59}{0.3}} & \baseline{\acc{61.87}{0.2}} & \baseline{\acc{71.50}{0.1}} & \baseline{\acc{70.17}{0.4}} & \baseline{\acc{64.62}{0.3}} \\
& Local training &
\baseline{\acc{71.36}{0.02}} & \baseline{\acc{38.67}{0.3}} & \baseline{\acc{81.25}{0.1}} & \baseline{\acc{76.27}{0.2}} &
\baseline{\acc{65.31}{0.5}} & \baseline{\acc{38.25}{0.7}} & \baseline{\acc{66.52}{0.3}} & \baseline{\acc{78.43}{0.3}} & \baseline{\acc{71.04}{0.2}} & \baseline{\acc{70.53}{0.6}} \\
\hline
\multirow{4}{*}{\rotatebox[origin=c]{90}{\textbf{FL}}}
  & FedAvg  & \worstcen{\acc{71.88}{0.1}} & \worstcen{\acc{48.44}{0.1}} & \worst{\acc{40.63}{0.2}} & \worst{\acc{54.24}{0.6}} &
\worst{\acc{55.71}{0.2}} & \worst{\acc{28.42}{0.5}} & \worst{\acc{40.25}{0.3}} & \worst{\acc{52.64}{0.2}} & \worst{\acc{54.15}{0.1}} & \worst{\acc{56.12}{0.2}} \\
  & FedProx & \worstcen{\acc{73.44}{0.2}} & \worstcen{\acc{52.00}{0.2}} & \worst{\acc{68.75}{0.4}} & \best{\acc{79.66}{0.4}} &
\worst{\acc{59.41}{0.2}} & \second{\acc{35.74}{0.4}} & \worst{\acc{48.82}{0.4}} & \worst{\acc{55.37}{0.1}} & \worst{\acc{56.82}{0.5}} & \worst{\acc{59.17}{0.2}} \\
  & FedCT & \worstcen{\acc{73.96}{0.1}} & \worstcen{\acc{57.21}{0.2}} & \worst{\acc{68.73}{0.01}} & \worst{\acc{70.31}{0.02}} &
\worst{\acc{61.53}{0.3}} & \second{\acc{35.19}{0.01}} & \second{\acc{64.73}{0.03}} & \worst{\acc{60.82}{0.01}} & \best{\acc{71.85}{0.02}} & \second{\acc{69.25}{0.01}} \\
  & FedBN & \best{\acc{75.39}{0.01}} & \worstcen{\acc{58.13}{0.01}} & \worst{\acc{78.54}{0.2}} & \worstcen{\acc{78.23}{0.8}} &
\worstcen{\acc{69.45}{0.3}} & \worst{\acc{38.01}{0.1}} & \best{\acc{68.12}{0.2}} & \best{\acc{79.21}{0.2}} & \best{\acc{76.20}{0.1}} & \second{\acc{69.23}{0.1}} \\
\hline
\multirow{6}{*}{\rotatebox[origin=c]{90}{\textbf{PFL}}}
  & Per-FedAvg & \worstcen{\acc{73.04}{0.1}} & \worstcen{\acc{51.81}{0.5}} & \worst{\acc{69.22}{0.3}} & \worstcen{\acc{77.58}{0.01}} &
\second{\acc{68.42}{0.01}} & \second{\acc{36.21}{0.2}} & \worst{\acc{60.49}{0.2}} & \second{\acc{72.63}{0.1}} & \second{\acc{70.84}{0.3}} & \second{\acc{68.16}{0.3}} \\
  & Ditto & \best{\acc{75.30}{0.01}} & \worstcen{\acc{57.91}{0.3}} & \worst{\acc{78.39}{0.02}} & \second{\acc{78.39}{0.1}} &
\best{\acc{70.97}{0.01}} & \best{\acc{39.13}{0.01}} & \best{\acc{67.31}{0.02}} & \best{\acc{80.33}{0.03}} & \best{\acc{77.35}{0.01}} & \best{\acc{73.14}{0.03}} \\
  & pFedMe & \worst{\acc{70.83}{0.3}} & \worstcen{\acc{49.78}{0.1}} & \worst{\acc{75.00}{0.03}} & \worst{\acc{64.41}{0.01}} &
\worstcen{\acc{67.21}{0.1}} & \second{\acc{37.42}{0.3}} & \second{\acc{65.17}{0.2}} & \second{\acc{75.24}{0.2}} & \best{\acc{74.19}{0.1}} & \second{\acc{68.93}{0.3}} \\
  & APFL & \worst{\acc{71.30}{0.01}} & \worstcen{\acc{39.05}{0.06}} & \worst{\acc{50.85}{0.2}} & \worst{\acc{69.63}{0.1}} &
\worstcen{\acc{68.73}{0.1}} & \second{\acc{38.05}{0.3}} & \best{\acc{67.39}{0.3}} & \best{\acc{79.14}{0.01}} & \best{\acc{77.42}{0.1}} & \best{\acc{71.85}{0.2}} \\
  & FedPHP & \worst{\acc{70.63}{0.5}} & \worstcen{\acc{40.13}{0.04}} & \worst{\acc{51.78}{0.01}} & \worst{\acc{72.74}{0.02}} &
\worst{\acc{65.29}{0.4}} & \second{\acc{36.32}{0.3}} & \second{\acc{66.01}{0.5}} & \second{\acc{77.03}{0.2}} & \best{\acc{75.28}{0.6}} & \second{\acc{70.11}{0.1}} \\
& PerFed-CKT & \worst{\acc{71.26}{0.1}} & \worstcen{\acc{46.80}{0.3}} & \worst{\acc{74.22}{0.2}} & \worst{\acc{ 73.50}{0.02}} &
\worstcen{\acc{67.49}{0.2}} & \second{\acc{37.41}{0.1}} & \second{\acc{62.83}{0.5}} & \second{\acc{ 72.45}{0.1}} & \worst{\acc{65.39}{0.2}} & \worst{\acc{62.59}{0.1}} \\
\hline
& \pfedct & \best{\acc{\mathbf{80.21}}{0.01}} & \best{\acc{\mathbf{60.00}}{0.02}} & \best{\acc{\mathbf{81.25}}{0.02}} & \best{\acc{\mathbf{83.05}}{0.1}} &
\best{\acc{\mathbf{71.36}}{0.1}} & \best{\acc{\mathbf{41.59}}{0.2}} & \best{\acc{\mathbf{69.38}}{0.4}} & \best{\acc{\mathbf{84.27}}{0.1}} & \best{\acc{\mathbf{79.25}}{0.3}} & \best{\acc{\mathbf{75.03}}{0.2}} \\
\hline
\end{tabular}
}
\label{tab:feature-shift}
\end{table}

\paragraph{Feature Shift:}
To evaluate robustness to heterogeneous input distributions, we test on feature shift scenarios using the Office-10 and DomainNet datasets. Here, each domain (e.g., "Webcam," "Sketch") acts as a client, sharing a common label space but having a unique data style. Table \ref{tab:feature-shift} shows that \pfedct consistently sets the state-of-the-art on all domains. On the complex DomainNet benchmark, it achieves the highest accuracy across all six domains, outperforming specialized methods like Ditto and FedBN. This demonstrates that the dynamic weighting and confidence-based aggregation are not limited to label skew; they effectively manage domain-specific features, allowing clients to learn from each other while preserving their specialized knowledge.

\paragraph{Hybrid Distribution (Label Skew + Feature Shift):}
We now consider the most challenging scenario: a hybrid of label skew and feature shift. To simulate this, we partition each domain in DomainNet and Office-10 into $5$ clients, each assigned only $2$ of the $10$ classes. This results in $30$ highly heterogeneous clients for DomainNet and 20 for Office-10. In this demanding setup, we evaluate both our confidence mechanisms: the class-frequency heuristic (\pfedct-W) and the uncertainty-based score (\pfedct-U).

The results in Table \ref{tab:hybrid-settings} confirm the superiority of our approach. With both AlexNet and ViT architectures, \pfedct variants significantly outperform all baselines. On Office-10, for instance, \pfedct-U achieves $0.8943$ accuracy, a remarkable improvement over the next best baseline, Ditto ($0.8063$). One can note that centralized training is worse than local training due to the highly heterogeneous setting, meaning that a single golbal model cannot fit all clients effectively. 

Interestingly, both the simple class-frequency heuristic and the more complex uncertainty-based score yield similarly strong results. This suggests that in settings with extreme label skew, class frequency serves as a powerful and efficient proxy for model expertise.

Taken together, these results validate that \pfedct's principled approach to adaptive, expert-aware collaboration enables it to deliver state-of-the-art performance, consistently outperforming strong baselines in diverse and realistic non-IID settings.

\begin{table}[t]
    \centering
    \caption{Average test accuracy (in \%) on the DomainNet and Office-10 dataset in hybrid settings for $m=30$ clients on DomainNet and $m=20$ on Office-10. Color map: see Table~\ref{tab:label-skew}.}
    \resizebox{0.8\textwidth}{!}{
    \begin{tabular}{l|l|l|l|l}
        \hline
        & \textbf{Method} & \textbf{DomainNet} & \textbf{DomainNet (ViT)} & \textbf{Office-10} \\
        \hline
        & Centralized   & \baseline{66.24 \tiny (0.4)} & \baseline{68.25 \tiny (0.2)} & \baseline{40.92 \tiny (0.6)} \\
        & Local training& \baseline{84.64 \tiny (0.1)} & \baseline{84.92 \tiny (0.3)} & \baseline{86.79 \tiny (0.4)} \\
        \hline
        \multirow{5}{*}{\rotatebox[origin=c]{90}{\textbf{FL}}}
          & FedAvg  & \worst{31.00 \tiny (0.8)} & \worst{33.28 \tiny (0.5)} & \worst{37.25 \tiny (0.8)} \\
          & FedProx & \worst{55.23 \tiny (0.1)} & \worst{57.18 \tiny (0.3)} & \second{58.39 \tiny (0.3)} \\
          & FedCT   & \worst{56.38 \tiny (0.01)} & \worst{67.52 \tiny (0.02)} & \worst{59.42 \tiny (0.02)} \\
          & FedBN   & \second{71.54 \tiny (0.3)} & \second{70.39 \tiny (0.4)} & \second{75.48 \tiny (0.3)} \\
        \hline
        \multirow{6}{*}{\rotatebox[origin=c]{90}{\textbf{PFL}}}
          & Per-FedAvg & \second{72.48 \tiny (0.4)} & \second{73.19 \tiny (0.3)} & \second{71.92 \tiny (0.5)} \\
          & Ditto      & \second{81.47 \tiny (0.01)} & \second{83.82 \tiny (0.02)} & \second{80.63 \tiny (0.01)} \\
          & pFedMe     & \second{75.21 \tiny (0.5)} & \second{76.81 \tiny (0.8)} & \second{74.83 \tiny (0.7)} \\
          & APFL       & \second{80.59 \tiny (0.3)} & \second{83.27 \tiny (0.5)} & \second{80.91 \tiny (0.1)} \\
          & FedPHP     & \second{78.25 \tiny (0.6)} & \second{77.31 \tiny (0.7)} & \second{76.36 \tiny (0.4)} \\
          & PerFed-CKT & \second{79.24 \tiny (0.4)} & \second{80.16 \tiny (0.2)} & \second{82.49 \tiny (0.1)} \\
        \hline
        & \pfedct(W)    & \best{\textbf{87.44 \tiny (0.02)}} & \best{\textbf{88.52 \tiny (0.2)}} & \best{\textbf{89.06 \tiny (0.01)}} \\
        & \pfedct(U)    & \best{\textbf{88.36 \tiny (0.01)}} & \best{\textbf{87.35 \tiny (0.1)}} & \best{\textbf{89.43 \tiny (0.03)}} \\
        %& DP-\pfedct(U) & \best{\textbf{85.93 \tiny (0.1)}} & \best{\textbf{85.74 \tiny (0.1)}} & \best{\textbf{87.19 \tiny (0.04)}} \\
        \hline
    \end{tabular}
    }
    \label{tab:hybrid-settings}
\end{table}

\paragraph{The Effect of the Unlabeled Dataset:}
\pfedct relies heavily on a shared unlabeled dataset $|U|$. To understand how sensitive \pfedct is to the characteristics of this dataset, we conducted a study on the effect of the size and distribution of this dataset. We simulated varying degrees of skew by sampling $|U|$ (with a fixed size of 3,000) using a Dirichlet distribution. We tested concentration parameters $\alpha=\{1,0.7,0.5,0.3,0.1\}$, where $\alpha=1$ corresponds to a perfectly IID distribution and lower values induce increasingly severe skew.
As shown in Fig.~\ref{fig:u_FedMOSIC} and Fig.~\ref{fig:u_distribution_main}, performance degrades as the public dataset becomes more skewed, especially at low $\alpha$ (e.g., $0.3, 0.1$) where some classes are missing. However, a key finding is that \pfedct never performs worse than the local baseline. This highlights the robustness of the adaptive aggregation scheme: when the global signal is unhelpful, the dynamic weight $\lambda$ steers clients toward local training, acting as a fail-safe. More details are provided in App.~\ref{app:additional-exp}.
\begin{figure}[H]
\begin{minipage}[H]{0.48\textwidth}
    \centering
    \includegraphics[width=\linewidth]{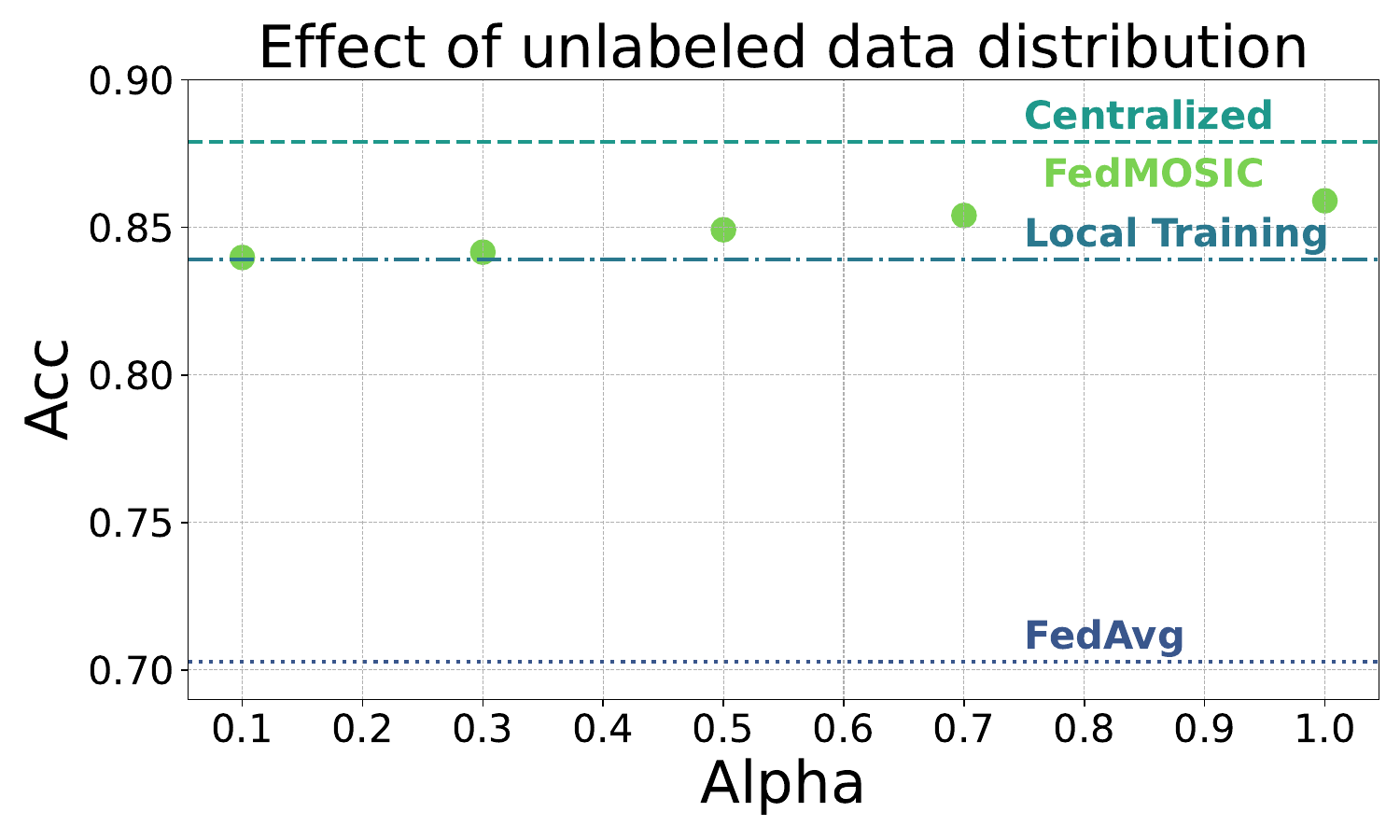}
    \caption{ Average test accuracy of \pfedct on CIFAR-10 under different distribution of $U$. }
    \label{fig:u_FedMOSIC}
\end{minipage}\hfill
\begin{minipage}[H]{0.48\textwidth}
    \centering
    \includegraphics[width=\linewidth]{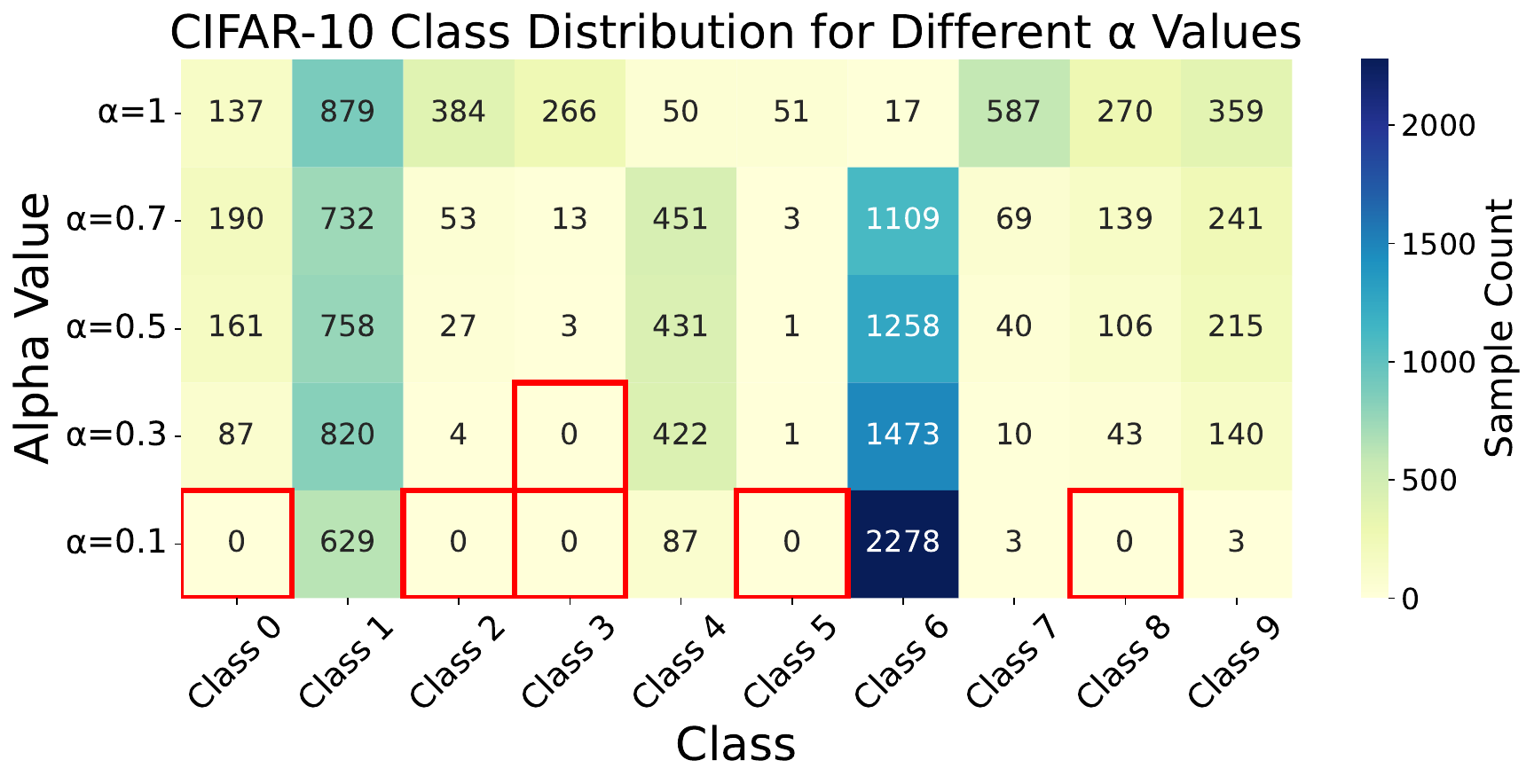}
    \caption{Class distribution of $U$ under different values of alpha.}
    \label{fig:u_distribution_main}
\end{minipage}
\end{figure}

\section{Discussion and Conclusion}
\label{sec:discussion}

Personalized Federated Learning (PFL) aims to address data heterogeneity by tailoring models to client-specific distributions. Yet, as we have demonstrated, many existing approaches fall short of their promise, often failing to outperform even local training or centralized baselines. This raises fundamental concerns about the core premise of collaboration in personalized federated learning.

We argue that meaningful personalization in federated learning requires more than per-client modeling: it must involve adaptive, data-specific collaboration. In particular, effective PFL methods should support example-level decision-making, allowing clients to modulate the degree and direction of collaboration based on local context and per-sample reliability. Without this level of adaptivity, personalization risks becoming a superficial modification of global training.

\pfedct is one concrete instantiation of this principle. It enables example-level collaboration through dynamic loss weighting and confidence-based aggregation over a shared unlabeled dataset. Unlike prior methods that personalize only at the client level, \pfedct allows each client to adapt both how much and whom to trust, based on the alignment between public and private data.

Empirical results across a diverse set of non-IID scenarios support the effectiveness of this approach. In the hybrid scenario, which combines label skew and feature shift, \pfedct outperforms all competitors and baselines by a wide margin. In the feature shift scenarios, it again surpasses all methods across most domains, often with substantial gains. In the label skew setting, \pfedct consistently achieves the best performance for the pathological non-IID scenario, though with very narrow margins, in particular with respect to local training. In the practical non-IID scenario with milder heterogeneity, centralized training performs best, as expected. Yet, traditional federated learning methods fall short, being outperformed by several PFL approaches, including \pfedct.

These results illustrate both the strengths and limitations of personalized FL. One limitation is that, particularly in the label skew setting, the advantage over strong local baselines can be modest. Such scenarios, especially the pathological non-IID one, raise the question of whether collaboration is truly justified, and whether evaluation setups that favor strong local baselines but show weak global benefit are well-posed. We therefore emphasize the need for more meaningful benchmarks: scenarios where collaboration has a clear potential upside, and where the evaluation criteria capture the practical value of federated interaction, not just statistical differences. That said, \pfedct demonstrates that adaptive and data-aware collaboration is both feasible and effective. Across our experiments, it outperforms both local and centralized baselines in most settings, supporting its robustness and practical utility.

While \pfedct represents a principled and practically validated advance in personalized federated learning, it also opens new directions for future work. A key limitation is the assumption of a public unlabeled dataset. Although such datasets exist in many domains, e.g., healthcare, vision, and language, it remains an open question how to extend this paradigm when such data are limited or unavailable. Developing mechanisms for privacy-preserving dataset synthesis, or leveraging foundation models for public data distillation, could further broaden the applicability of our framework.

\bibliography{iclr2026_conference}
\bibliographystyle{iclr2026_conference}

\vfill
\pagebreak
\appendix

\section{Proof of Theorem}
\label{app-proof-theorm}
In the following, we proof Proposition~\ref{theorem-convergence}. For convenience, we restate the assumptions and proposition.

\setcounter{assumptions}{0}
\begin{assumptions}
The following conditions hold for each client $ i \in [m] $ at round $ t $:
\begin{enumerate}
    \item Each loss function $\mathcal{L}_i^{\mathrm{local}}$ and $\mathcal{L}_i^{\mathrm{global}, t}$ is $L(1+e)^{-1}$-smooth.
    \item The gradient estimator $g_i^t$ is unbiased and has bounded variance:
    \[
    \mathbb{E}[g_i^t] = \nabla \mathcal{L}_i^t(\theta_t), \quad \mathbb{E}[\|g_i^t - \nabla \mathcal{L}_i^t(\theta_t)\|^2] \leq \sigma^2.
    \]    
    \item The global loss has bounded gradients: $\|\nabla \mathcal{L}_i^{\mathrm{global}, t}(\theta)\| \leq G$ for all $\theta$ and $t$.
    \item The objective drift is bounded:
    \[
    |\mathcal{L}_i^{t+1}(\theta) - \mathcal{L}_i^t(\theta)| \leq \delta, \quad \forall \theta.
    \]
    \item The per-sample gradient variance is bounded:
    \begin{equation*}
    \begin{split}
    \mathbb{E}_{x\in D_i}\left[\left\|\nabla_\theta\ell(\theta,x,\hat{y}^t)-\nabla\mathcal{L}_i^{local,t}\right\|^2\right]\leq \overline\sigma^2\\
    \mathbb{E}_{x\in U}\left[\left\|\nabla_\theta\ell(\theta,x,\hat{y}^t)-\nabla\mathcal{L}_i^{global,t}\right\|^2\right]\leq \tilde\sigma^2\\
    \end{split}
    \end{equation*}
\end{enumerate}
\end{assumptions}
With these assumptions, \pfedct converges to a stationary point.

\setcounter{proposition}{0}
\begin{proposition}[Convergence of \pfedct]
Let each client's objective at round $t$ be
\[
\mathcal{L}_i^t(\theta) = \mathcal{L}_i^{\mathrm{local}}(\theta) + \lambda_i^t \mathcal{L}_i^{\mathrm{global}, t}(\theta),\text{ where }\lambda_i^t = \exp\left(-\frac{\mathcal{L}_i^{\mathrm{global}}(\theta_t) - \mathcal{L}_i^{\mathrm{local}, t}(\theta_t)}{\mathcal{L}_i^{\mathrm{local}, t}(\theta_t)}\right),
\]
and $\mathcal{L}_i^{\mathrm{global}, t}$ may change at each round due to pseudo-label updates. Under Assumptions 1-5, for a fixed step size $0<\eta\leq(2L)^{-1}$ and $\min_i |D_i|=d$, after $T$ rounds of \pfedct, it holds that
\[
\frac{1}{T} \sum_{t=0}^{T-1} \mathbb{E}[\|\nabla \mathcal{L}_i^t(\theta_t)\|^2]
\leq \frac{4L\left(\mathcal{L}_i^0 - \mathcal{L}_i^*\right)}{T} + \frac{\overline\sigma^2}{2Ld} +  \frac{e^2\tilde\sigma^2}{2L|U|} + 2\delta\enspace .
\]
\end{proposition}

\begin{proof}
Since $\mathcal{L}_i^{\mathrm{local}}$ and $\mathcal{L}_i^{\mathrm{global}, t}$ are $L(1+e)^{-1}$-smooth, and since during  optimization steps $\lambda_i^t<e$ is fixed, the Lipschitz constant of $\mathcal{L}_i^t$ is
\[
L(1+e)^{-1}+\lambda_i^tL(1+e)^{-1} \leq L(1+e)^{-1}+eL(1+e)^{-1} = L\enspace. 
\]
Thus, the standard descent lemma~\citep{bottou2018optimization} gives:
\[
\mathbb{E}[\mathcal{L}_i^t(\theta_{t+1})] \leq \mathbb{E}[\mathcal{L}_i^t(\theta_t)] - \eta \mathbb{E}[\|\nabla \mathcal{L}_i^t(\theta_t)\|^2] + \frac{L\eta^2}{2} \mathbb{E}[\|g_i^t\|^2].
\]
To bound $\mathbb{E}[\|g_i^t\|^2]$, expand
\[
\mathbb{E}[\|g_i^t\|^2] = \mathbb{E}[\|g_i^t - \nabla \mathcal{L}_i^t(\theta_t) + \nabla \mathcal{L}_i^t(\theta_t)\|^2] 
\leq 2\sigma^2 + 2\mathbb{E}[\|\nabla \mathcal{L}_i^t(\theta_t)\|^2],
\]
and substitute into the descent inequality to obtain
\[
\mathbb{E}[\mathcal{L}_i^t(\theta_{t+1})] \leq \mathbb{E}[\mathcal{L}_i^t(\theta_t)] - \eta \mathbb{E}[\|\nabla \mathcal{L}_i^t(\theta_t)\|^2] + L\eta^2\left(\sigma^2 + \mathbb{E}[\|\nabla \mathcal{L}_i^t(\theta_t)\|^2]\right).
\]
Rearranging terms yields
\[
\mathbb{E}[\mathcal{L}_i^t(\theta_{t+1})] \leq \mathbb{E}[\mathcal{L}_i^t(\theta_t)] - \eta(1 - L\eta) \mathbb{E}[\|\nabla \mathcal{L}_i^t(\theta_t)\|^2] + L\eta^2\sigma^2.
\]
This step requires $\eta \leq (2L)^{-1} < L^{-1}$ to ensure that the coefficient $(1 - L\eta)$ is positive.
We now account for the fact that the function changes between rounds, i.e., 
\[
\mathbb{E}[\mathcal{L}_i^{t+1}(\theta_{t+1})] \leq \mathbb{E}[\mathcal{L}_i^t(\theta_{t+1})] + \delta,
\]
which gives
\[
\mathbb{E}[\mathcal{L}_i^{t+1}(\theta_{t+1})] \leq \mathbb{E}[\mathcal{L}_i^t(\theta_t)] - \eta(1 - L\eta) \mathbb{E}[\|\nabla \mathcal{L}_i^t(\theta_t)\|^2] + L\eta^2\sigma^2 + \delta.
\]
Summing from $t=0$ to $T-1$ and rearranging yields
\[
\frac{1}{T} \sum_{t=0}^{T-1} \mathbb{E}[\|\nabla \mathcal{L}_i^t(\theta_t)\|^2]
\leq \frac{\mathcal{L}_i^0 - \mathcal{L}_i^T}{(1 - L\eta)\eta T} + \frac{L\eta^2\sigma^2}{1 - L\eta} + \frac{\delta}{1 - L\eta}.
\]
Denoting the minimum loss as $\mathcal{L}_i^*$, i.e., $\forall t$, $\mathcal{L}_i^t\geq \mathcal{L}_i^*$ yields the formal result
\[
\frac{1}{T} \sum_{t=0}^{T-1} \mathbb{E}[\|\nabla \mathcal{L}_i^t(\theta_t)\|^2]
\leq \frac{\mathcal{L}_i^0 - \mathcal{L}_i^*}{(1 - L\eta)\eta T} + \frac{L\eta^2\sigma^2}{1 - L\eta} + \frac{\delta}{1 - L\eta}.
\]
Since $((1 - L\eta)\eta)^{-1}$, $L\eta^2/(1-L\eta)$, and $(1 - L\eta)^{-1}$ have a maximum at $(2L)^{-1}$ for $\eta\leq(2L)^{-1}$, we can upper bound this by
\[
\frac{1}{T} \sum_{t=0}^{T-1} \mathbb{E}[\|\nabla \mathcal{L}_i^t(\theta_t)\|^2]
\leq \frac{4L\left(\mathcal{L}_i^0 - \mathcal{L}_i^*\right)}{T} + \frac{\sigma^2}{4L} + 2\delta\enspace .
\]
Since $g_i^t=g_i^{local,t}+\lambda_i^t g_i^{global,t}$, we decompose $\sigma^2$ in round $t$ at client $i$ as $2\overline\sigma^2+2(\lambda_i^t)^2\tilde\sigma^2$, and further bound
\begin{equation*}
    \begin{split}
        \sigma_{global}^2\leq & \frac{\mathbb{E}_{x\in D_i}\left[\left\|\nabla_\theta\ell(\theta,x,\hat{y}^t)-\nabla\mathcal{L}_i^{local,t}\right\|^2\right]}{\min_i |D_i|}\\
        & + \sup_{i,t}(\lambda_i^t)^2\frac{\mathbb{E}_{x\in U}\left[\left\|\nabla_\theta\ell(\theta,x,\hat{y}^t)-\nabla\mathcal{L}_i^{global,t}\right\|^2\right]}{|U|}\\
        \leq & \frac{2\overline\sigma^2}{d} + \frac{2e^2\tilde\sigma^2}{|U|}\enspace ,
    \end{split}
\end{equation*}
since $\sup_{i,t}(\lambda_i^t)^2=e^2$ and using Assumption 5. With this, we obtain obtain
\[
\frac{1}{T} \sum_{t=0}^{T-1} \mathbb{E}[\|\nabla \mathcal{L}_i^t(\theta_t)\|^2]
\leq \frac{4L\left(\mathcal{L}_i^0 - \mathcal{L}_i^*\right)}{T} + \frac{\overline\sigma^2}{2Ld} +  \frac{e^2\tilde\sigma^2}{2L|U|} + 2\delta\enspace .
\]
\end{proof}

\section{Additional Empirical Evaluation}
\label{app:additional-exp}

\paragraph{}

\paragraph{Robustness under Misleading Global Knowledge}
To further evaluate \pfedct's adaptivity, we conducted an experiment designed to test its behavior when the global consensus signal is actively misleading for a particular client. We constructed a scenario using CIFAR-10 dataset with $5$ clients, where client $0$ was assigned flipped labels so effectively training on corrupted data. This setup results in the global pseudo labels being systematically misaligned with this client's local distribution. As expected the client's local model suffers a significantly higher loss when trained using the global pseudo labels compared to its own data, leading to a near zero value of $\lambda$. This confirms the intended behavior of \pfedct: when the global signal is detrimental, the client autonomously reduces its reliance on it, effectively opting out of harmful collaboration. Fig.\ref{fig:global_flipped} illustrates this behavior by showing the divergence between global and local loss for the corrupted client (client $0$) in comparison to a non-corrupted one (client $1$). Fig.~\ref{fig:lambda_flipped} shows the evolution of the adaptive weight $\lambda$ across communication rounds for all $5$ clients.

\begin{figure}[H]
\begin{minipage}[H]{0.48\textwidth}
    \centering
    \includegraphics[width=\linewidth]{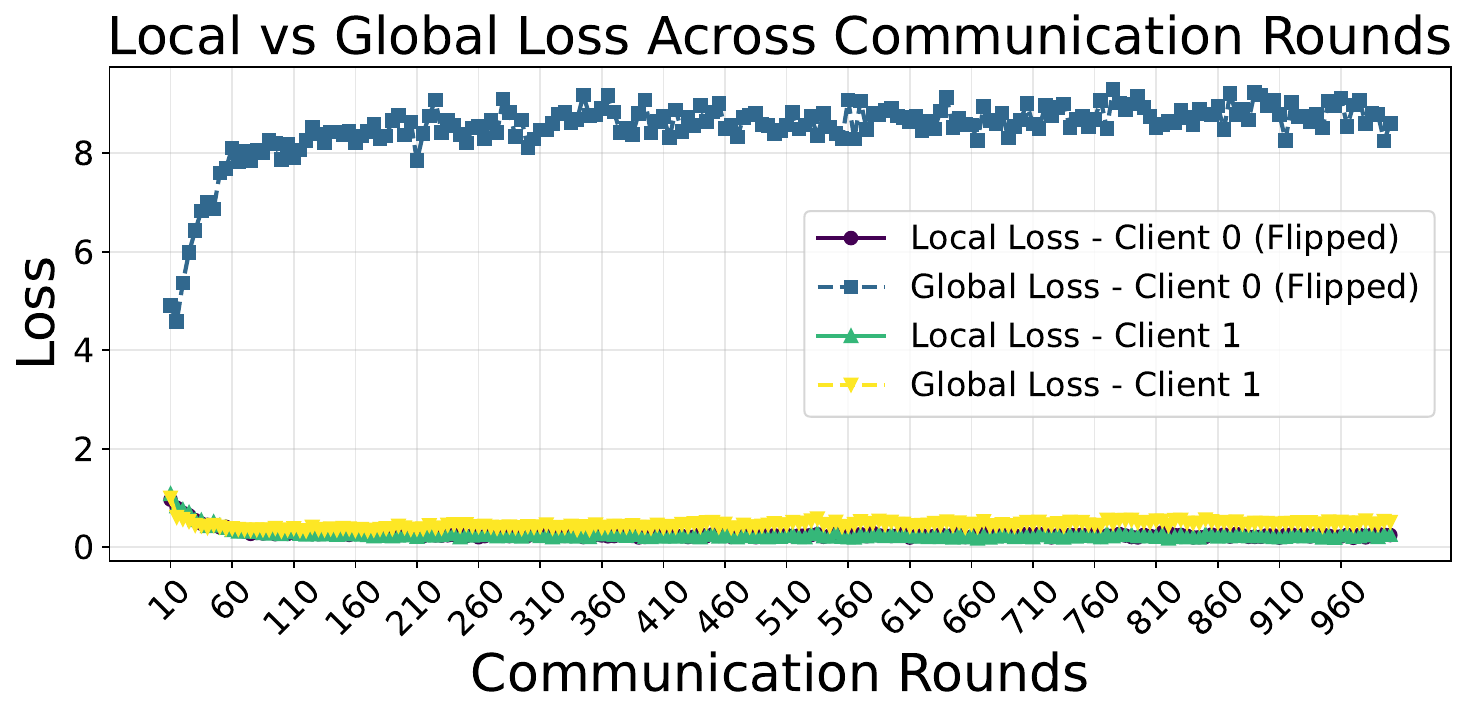}
    \caption{Local Vs Global loss across communication rounds on CIFAR-10.}
    \label{fig:global_flipped}
\end{minipage}\hfill
\begin{minipage}[H]{0.48\textwidth}
    \centering
    \includegraphics[width=\linewidth]{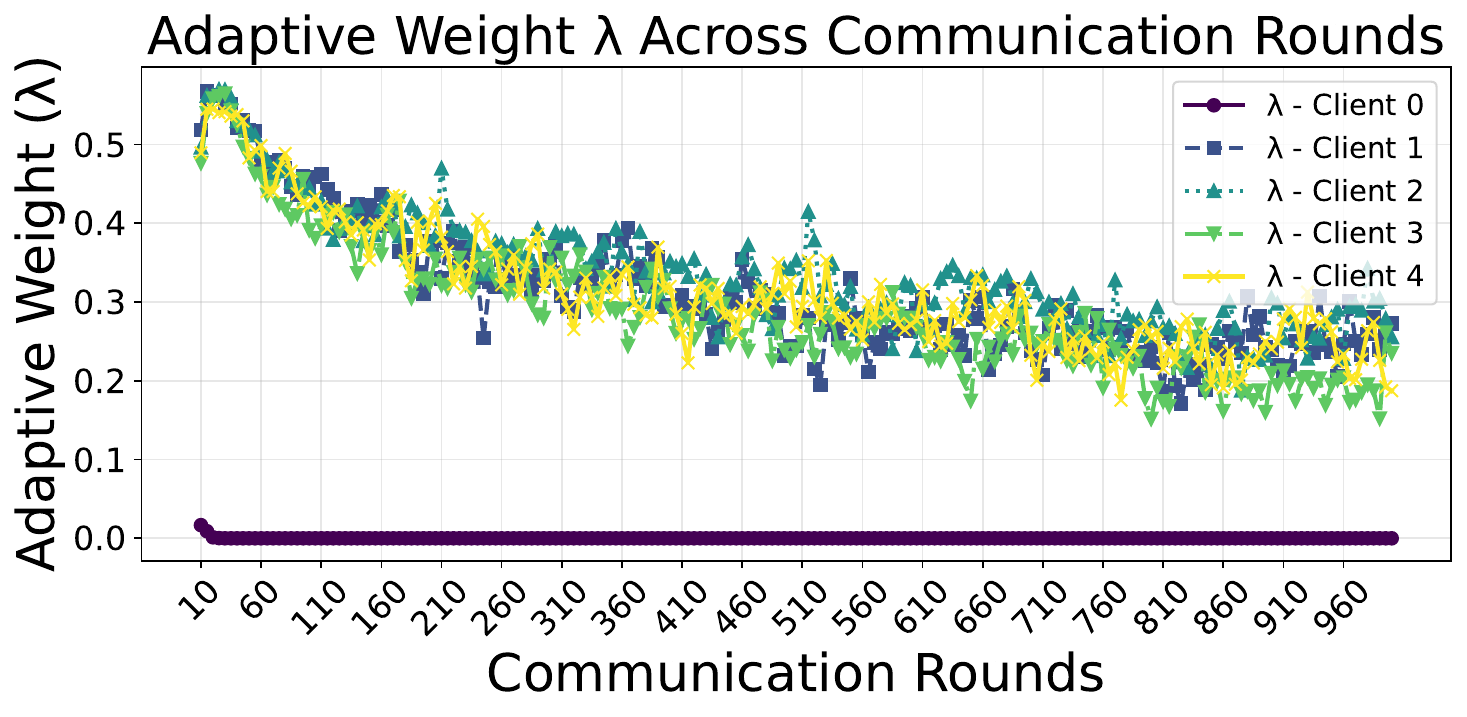}
    \caption{Adaptive weight $\lambda$ across communication rounds on CIFAR-10.}
    \label{fig:lambda_flipped}
\end{minipage}
\end{figure}

\paragraph{Personaliztion vs. Local training In Low-collaboration Regimes}

While \pfedct consistently archives the highest accuracy across both pathological and practical label skew settings (Table\ref{tab:label-skew}, the margin between its performance and that of local training is notably small. This observation raises a critical insight. In such scenarios, where each client's local distribution is highly disjoint and local alignment provides limited benefit, personalization through collaboration may be unnecessary or even detrimental. Indeed, \pfedct's adaptive mechanism reflects this reality. The per-client weighting strategy reduces reliance on the global information when it does not align with local data. This is evident in Fig.\ref{fig:loca_skew} and Fig.\ref{fig:global_skew}, which show that the global loss remains consistently higher than the local loss for many clients, leading to near zero value of the adaptive weight $\lambda$ as seen in Fig.\ref{fig:lambda_skew}. In such cases,\pfedct defaults to local training behavior, effectively opting out of collaboration when it offers no advantage. This reinforces the methods' robustness as it personalizes only when beneficial, and falls back to local training when collaboration yields little or a negative return. To ensure numerical stability in the computation of the adaptive coefficient \[\lambda_i^t = \exp\left(-\frac{\mathcal{L}_i^{\mathrm{global}}(\theta_t) - \mathcal{L}_i^{\mathrm{local}, t}(\theta_t)}{\mathcal{L}_i^{\mathrm{local}, t}(\theta_t)}\right), \], we add a small constant $\epsilon$ to the denominator to prevent division by zero.

\begin{figure}[H]
\begin{minipage}[H]{0.48\textwidth}
    \centering
    \includegraphics[width=\linewidth]{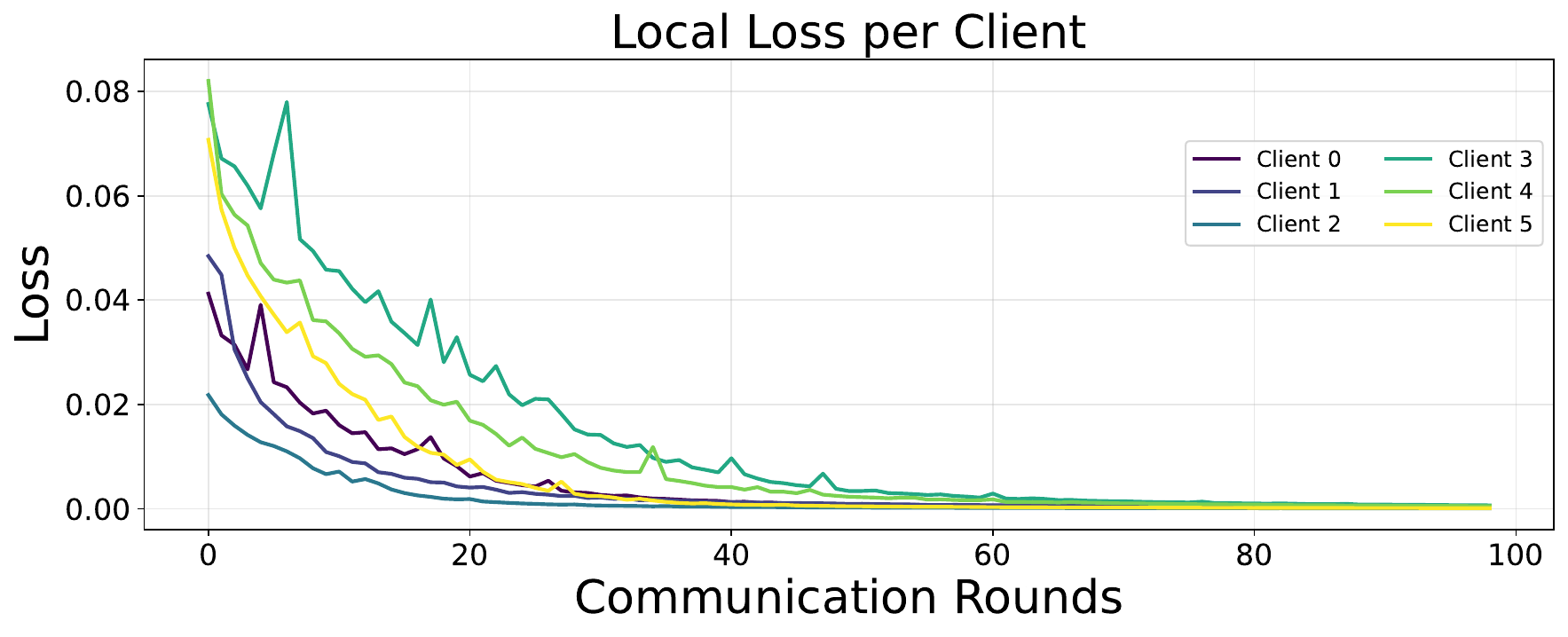}
    \caption{Local loss across communication rounds on Fashion-MNIST for the first $6$ clients.}
    \label{fig:loca_skew}
\end{minipage}\hfill
\begin{minipage}[H]{0.48\textwidth}
    \centering
    \includegraphics[width=\linewidth]{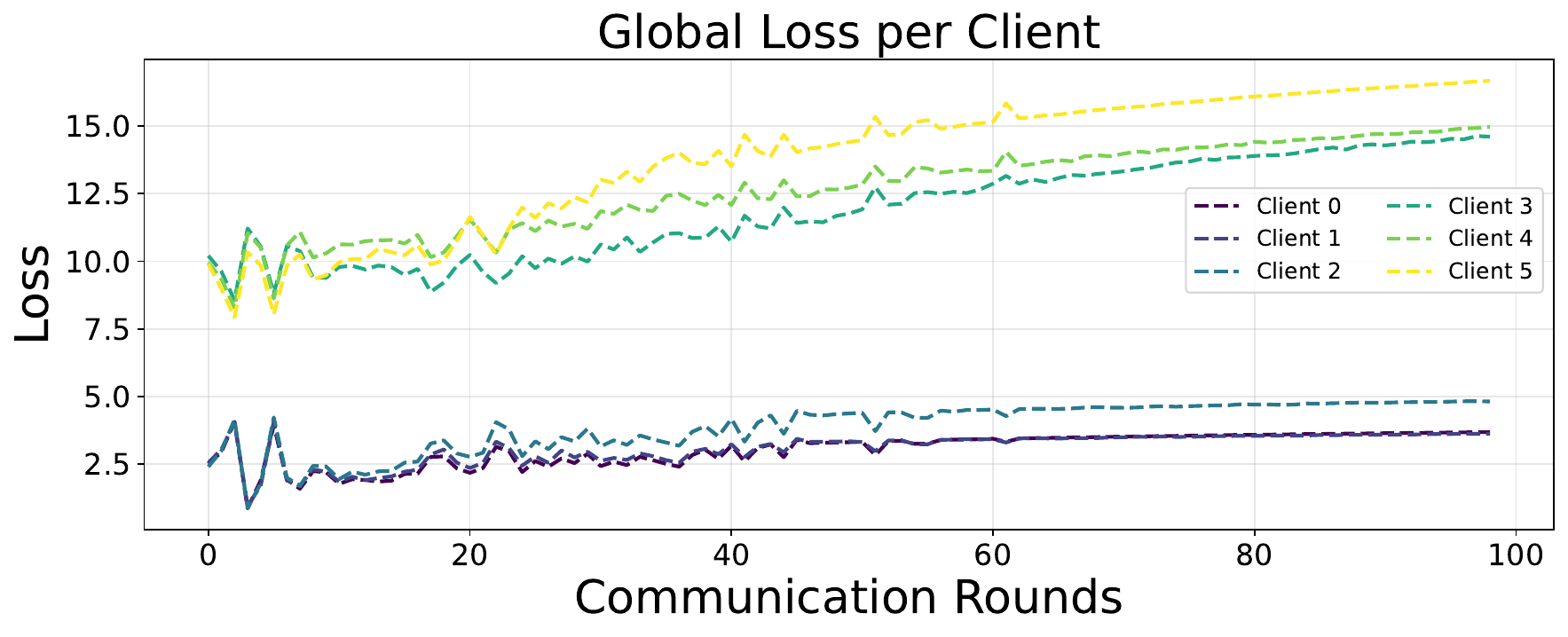}
    \caption{Global loss across communication rounds on Fashion-MINST for the first $6$ clients.}
    \label{fig:global_skew}
\end{minipage}
\end{figure}

\begin{figure}[H]
    \centering
    \includegraphics[width=0.5\linewidth]{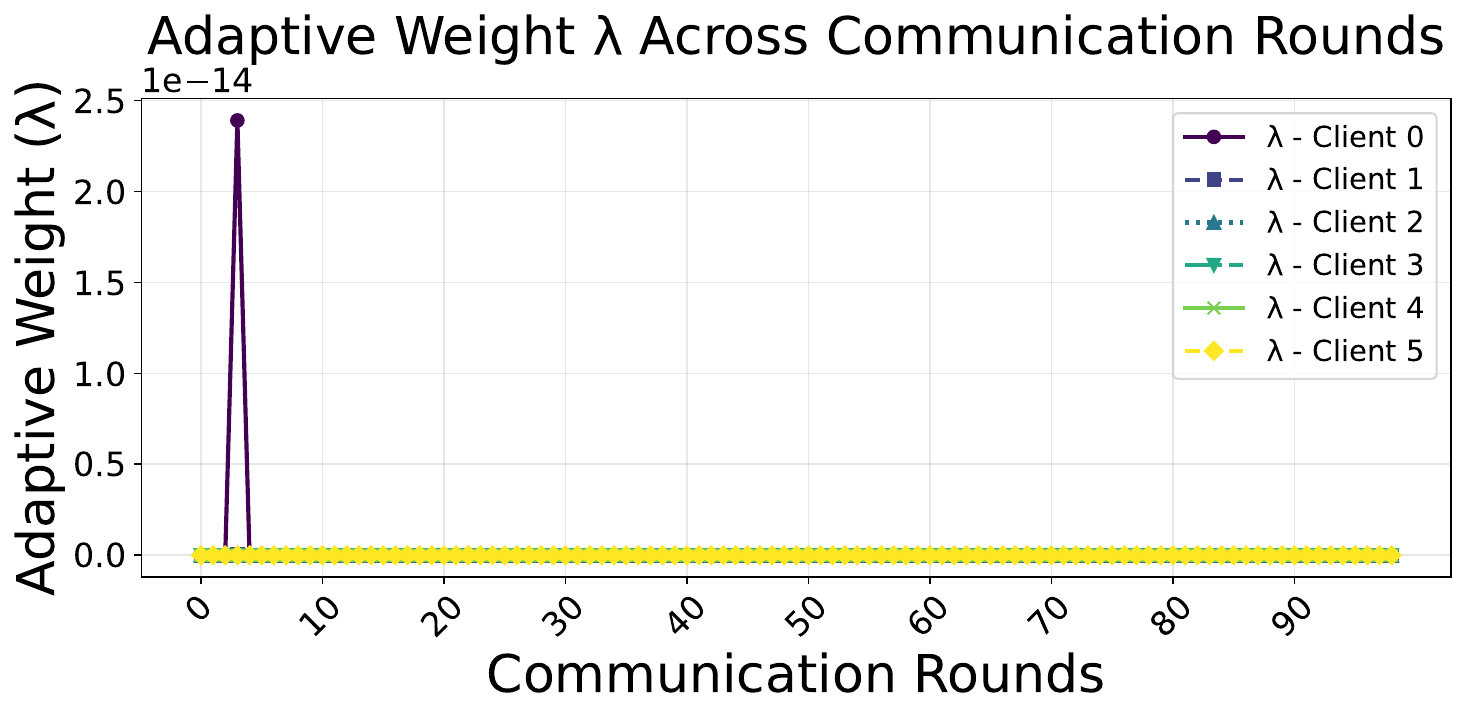}
    \caption{Adaptive weight $\lambda$ across communication rounds on Fashion-MINST for the first $6$ clients.}
    \label{fig:lambda_skew}
\end{figure}

\paragraph{The Effect of the Unlabeled Dataset}: As mensioned in sec.\ref{sec:experiments},\pfedct relies heavily on a shared, unlabeled public dataset $|U|$. To understand how sensitive \pfedct is to this dataset's characteristics, we conducted a study on CIFAR-10 dataset focusing on two critical questions: First, how does the amount of available data affect performance? Second, does it matter whether the class distribution is balanced (IID) or heavily skewed?

\subparagraph{Impact of Public Dataset Size}: We evaluated the performance of \pfedct using different sizes of the public unlabeled dataset, with $|U|$ set to $3000$,$2000$,$1000$,$500$ and $250$. For this experiment, the public dataset was always sampled in an IID fashion to ensure all classes were present. The results, summarized in Fig.\ref{fig:U_size}, show that the performance of \pfedct is remarkably stable. Even as the size of the public dataset is reduced by over $90\%$ ( from 3000 to 250 samples), the drop in final test accuracy is minimal. This finding suggests that the collaboration mechanism does not require a large volume of public unlabeled data. As long as a small class-representative set of examples is available, clients can effectively share knowledge and build high-quality personalized models.

\begin{figure}[H]
    \centering
    \includegraphics[width=0.5\linewidth]{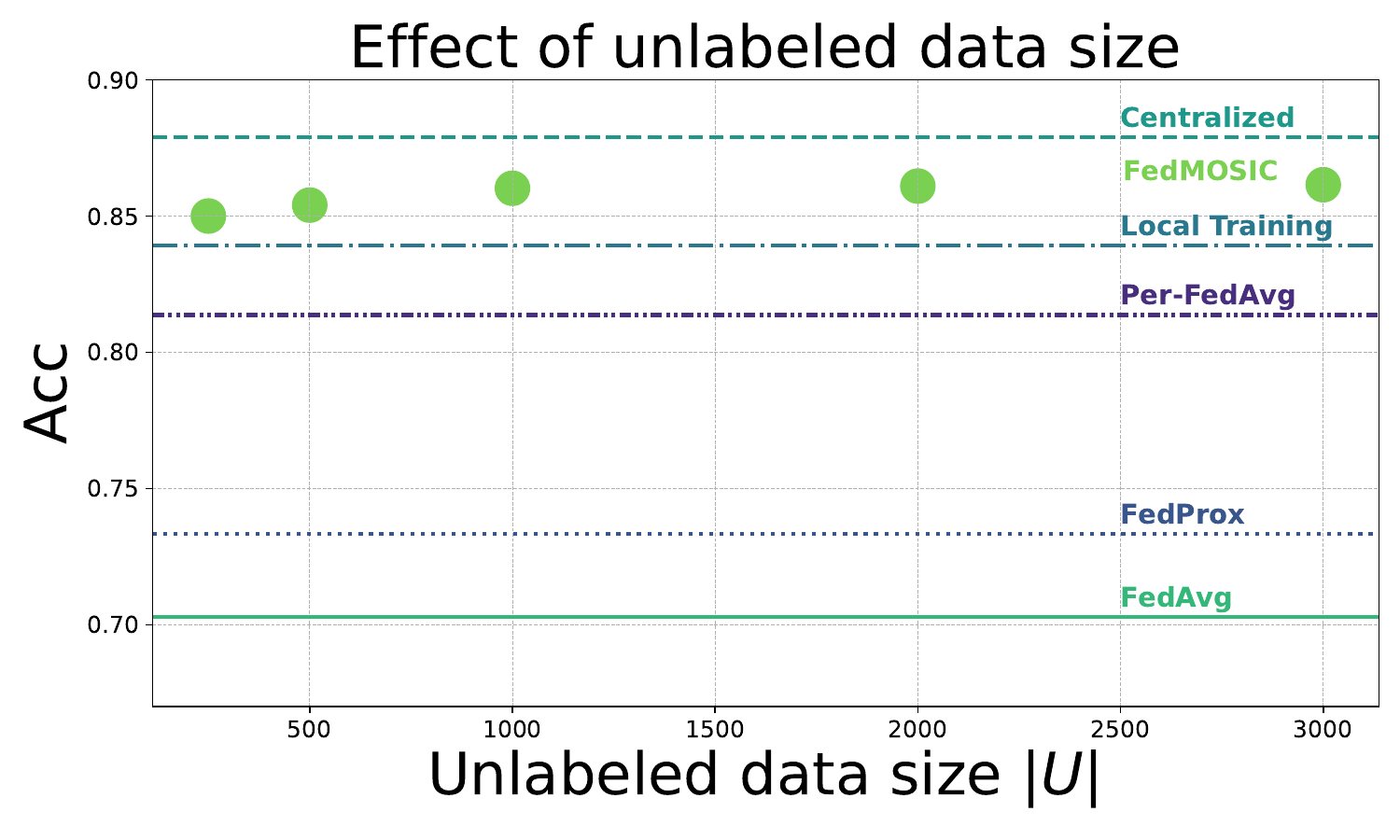}
    \caption{ Test accuracy (ACC) of \pfedct under different unlabeled dataset size $|U|$ }
    \label{fig:U_size}
\end{figure}

\subparagraph{Impact of Public Dataset Distribution}: Next, we studied the effect of the public unlabeled dataset distribution. We simulated varying degrees of distribution skew by sampling $|U|$ (with a fixed size of 3,000) using a Dirichlet distribution. We tested different values of the concentration parameter $\alpha={1,0.7,0.5,0.3,0.1}$, where $\alpha=1$ corresponds to a perfectly IID distribution and lower values induce increasingly severe skew.

As shown in Fig.\ref{} and Fig.\ref{fig:u_distribution}, we observe a degradition in performance as the public dataset become more skewed. The most significant drop occurs at very low $\alpha$ values (e.g.,0.3,0.1 ), where some classes are absent from $U$. In such cases, the global consensus offers no useful information for clients whose private data contains the missing classes. 

\begin{figure}[H]
\begin{minipage}[H]{0.48\textwidth}
    \centering
    \includegraphics[width=\linewidth]{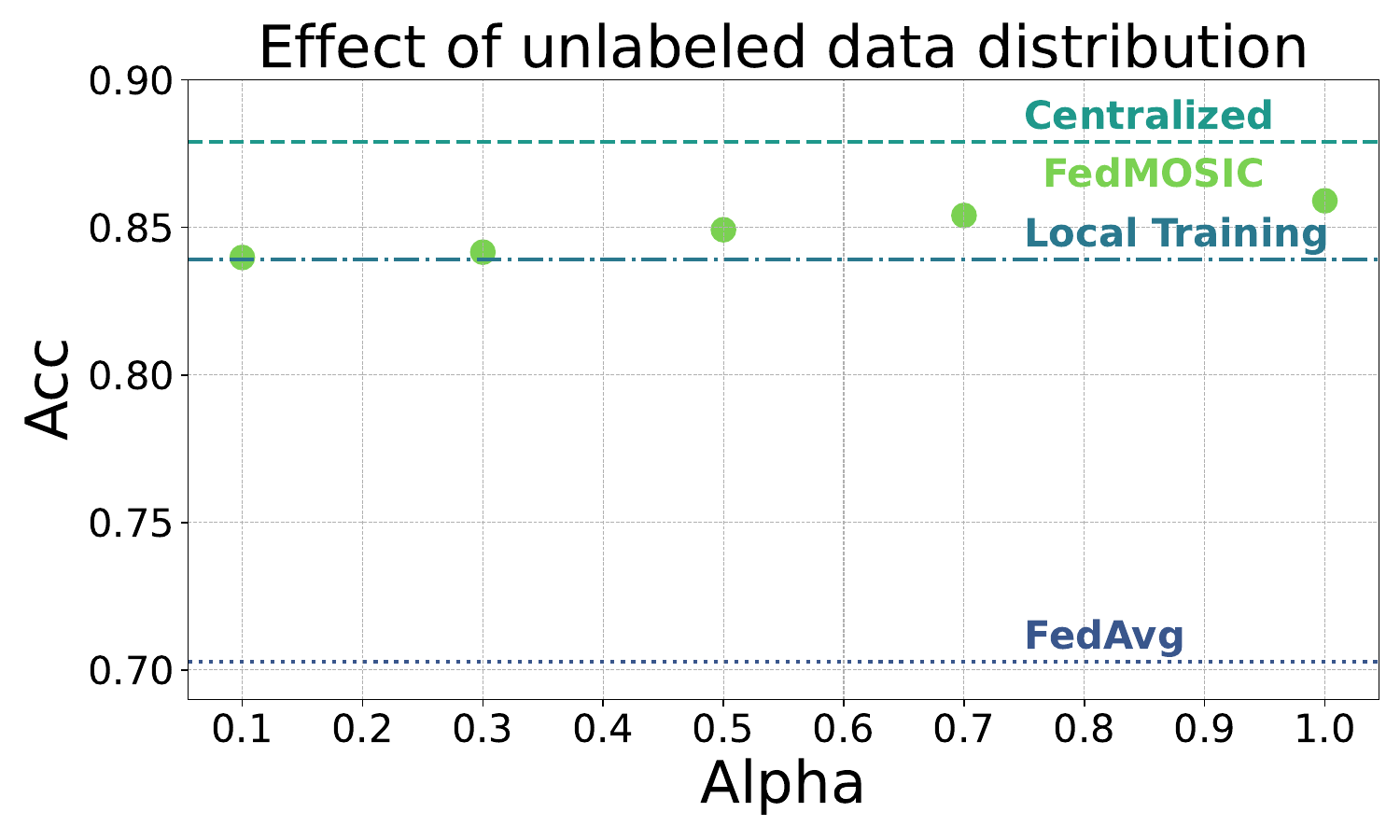}
    \caption{Test accuracy (ACC) of \pfedct under different distribution of $U$. }
    \label{}
\end{minipage}\hfill
\begin{minipage}[H]{0.48\textwidth}
    \centering
    \includegraphics[width=\linewidth]{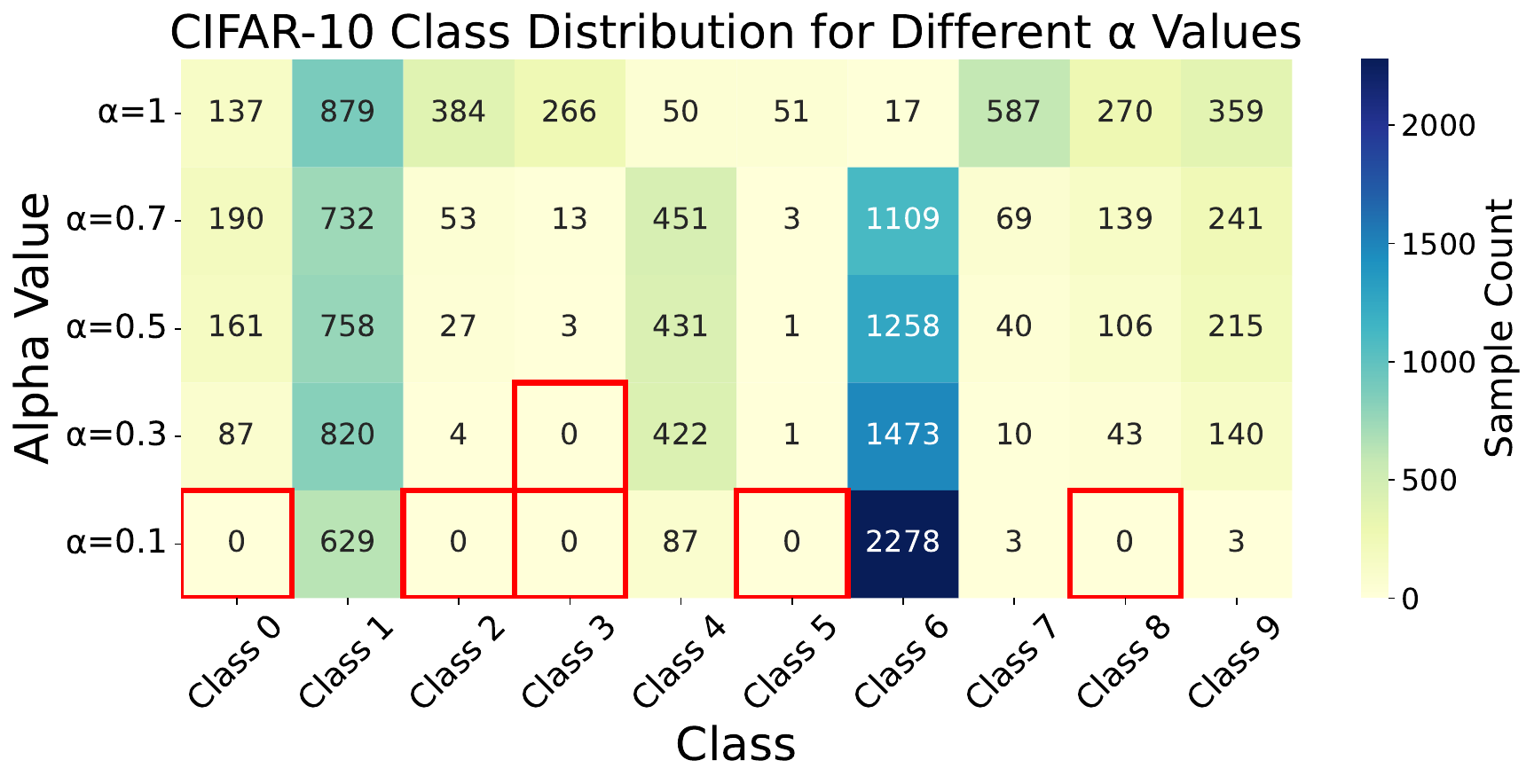}
    \caption{Class distributiuon of $U$ under different values of alpha.}
    \label{fig:u_distribution}
\end{minipage}
\end{figure}

However, the most crucial finding is that the performance of \pfedct never drops below the local training baseline. This demonstrates the robustness of the adaptive aggregation scheme. When the global signal becomes irrelevant or misleading, the dynamic loss weight $\lambda$ automatically steers clients to disregard it, effectively defaulting to local training. This acts as a critical fail-safe, ensuring that collaboration is never actively detrimental, even when the public data is of poor quality.

\paragraph{A Note on the Byzantine Resilience of \pfedct}
Following the argument by \citep{jiang2020federated}, who show that federated semi-supervised learning with soft labels sharing (e.g., FedDistill) is more Byzantine resilient than \fedavg due to the bounded nature of the threat vector on the probability simplex, we argue that \pfedct exhibits similar ( if not stronger) resilience properties. Like FedCT \citep{abourayya2025little}, \pfedct relies on hard label sharing, further constraining the threat vector to a binary classification decision per example. Moreover, \pfedct incorporates confidence-based aggregation, which naturally downweights unreliable predictions. This mechanism provides an additional layer of robustness by reducing the influence of low confidence ( and potentially malicious) clients. While a formal analysis remains open, these properties suggest that \pfedct may be at least as Byzantine resilient as FedDistill and FedCT. Exploring this direction further is promising for future work.

\section{Details on Experiments}
\label{app:exp-details}
All experiments are conducted for a sufficient number of communication rounds until convergence, using three different random seeds. While the standard deviation across the three runs with different seeds is consistently small, this observation aligns with prior work \cite{zhang2023fedcp},\cite{zhang2023fedala}, \cite{zhang2023gpfl}.

\paragraph{Label Skew}

Fashion-Minst and CIFAR-10 datasets have been used for label skew experiments. In Fashion-Minst, we converted the raw grayscale $28\times28$ images into Pytorch tensors and normalized pixel values to the range $[-1,1]$ using a mean of $0.5$ and standard deviation of $0.5$. In CIFAR-10, we converted RGB $32\times32$ images into Pytorch tensors of shape $[3,32,32]$ and normalizes each color channel independently to the range of $[-1,1]$, using a mean of $0.5$ and standard deviation of $0.5$. The data is partitioned across $15$ clients. In a pathological non-IID setting, each client receives data from only $2$ out of $10$ classes. In a practical non-IID setting, data is distributed across $15$ clients using a Dirichlet distribution. This creates naturally overlapping, imbalanced label distributions among clients. Training data distribution of each scenario of CIFAR-10 are showing in Fig.\ref{fig:CIFAR10_path} and Fig.\ref{fig:CIFARO10_pract}.

\begin{figure}[H]
\begin{minipage}[H]{0.48\textwidth}
    \centering
    \includegraphics[width=\linewidth]{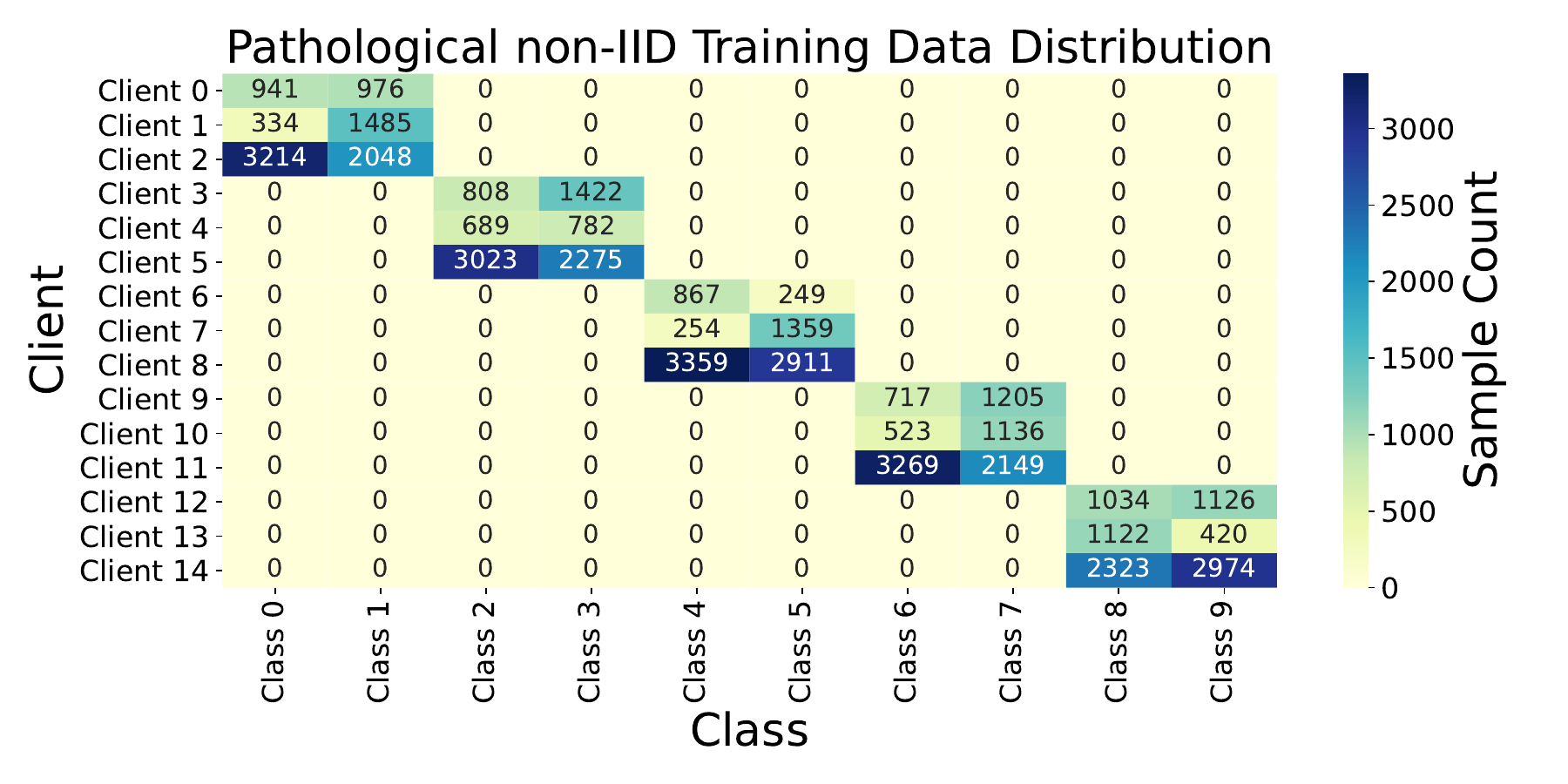}
    \caption{CIFAR-10 clients data distribution\\in Pathological non-IID setting}
    \label{fig:CIFAR10_path}
\end{minipage}\hfill
\begin{minipage}[H]{0.48\textwidth}
    \centering
    \includegraphics[width=\linewidth]{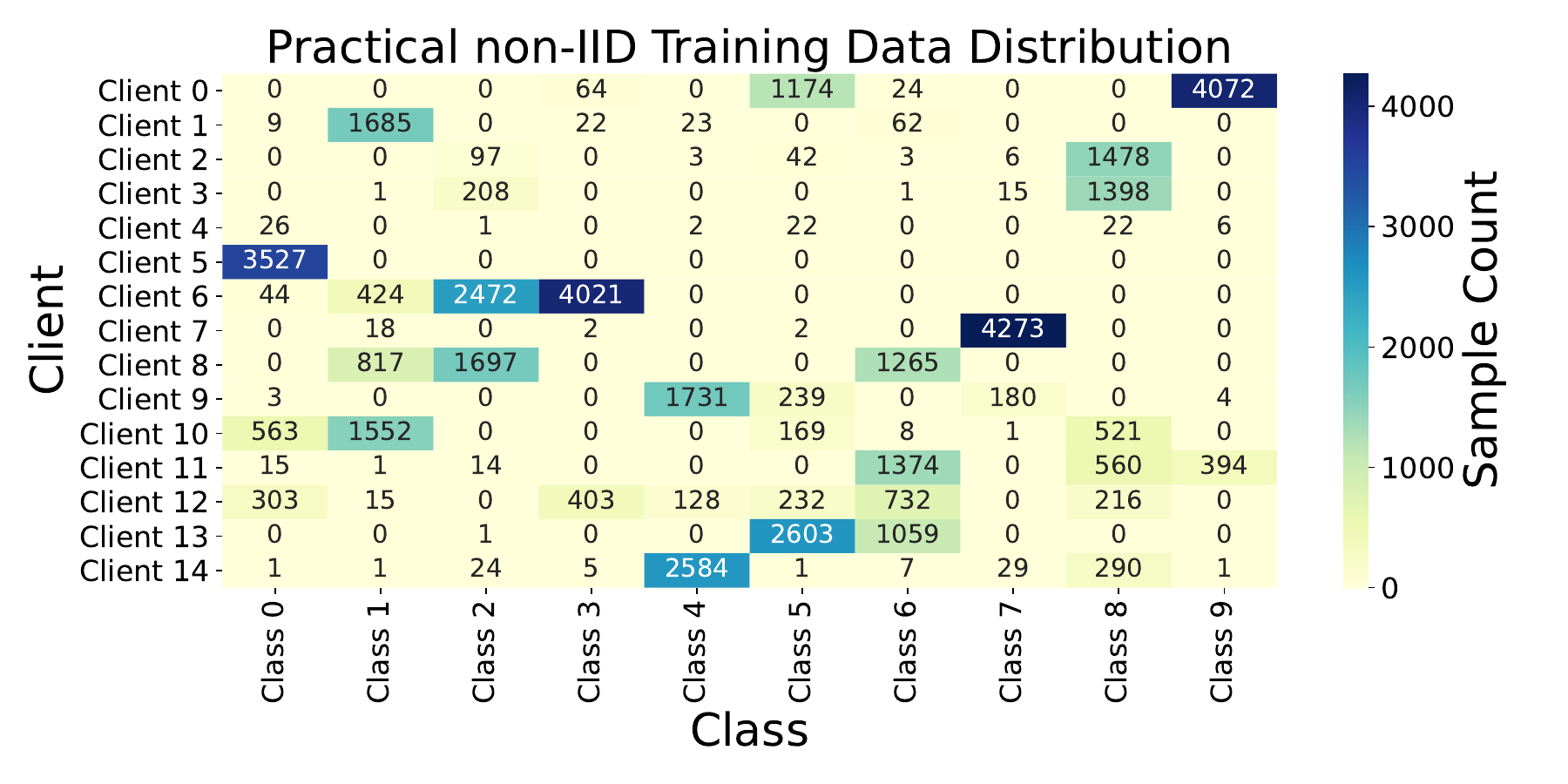}
    \caption{CIFAR-10 clients data distribution\\in Practical non-IID setting}
    \label{fig:CIFARO10_pract}
\end{minipage}
\end{figure}

\paragraph{Feature Shift}
we used the Office-10 and DomainNet datasets. For both, we adopt AlexNet as a neural network architecture. Input images are resized to $256\times 256\times 3$. Training is performed till convergence using the corss-entropy loss and Adam optimizer with learning rate of $10^{-2}$. We use a batch size of $32$ for Office-10 dataset and $64$ for DomainNet. For DomainNet, which originally contains $345$ categories, we restrict the label space to the top $10$ most frequent classes to reduce complexity, The selected categories are: \texttt{bird}, \texttt{feather}, \texttt{headphones}, \texttt{icecream}, \texttt{teapot}, \texttt{tiger}, \texttt{whale}, \texttt{windmill}, \texttt{wineglass}, \texttt{zebra}. For Office-10, each client get one of the $4$ domains and For DomainNet dataset, each client get one of the $6$ domains. The distribution of each client training data are showing in Fig.\ref{fig:DomainNet_feature} and Fig.\ref{fig:Office_feature}.

\begin{figure}[H]
\begin{minipage}[H]{0.48\textwidth}
    \centering
    \includegraphics[width=\linewidth]{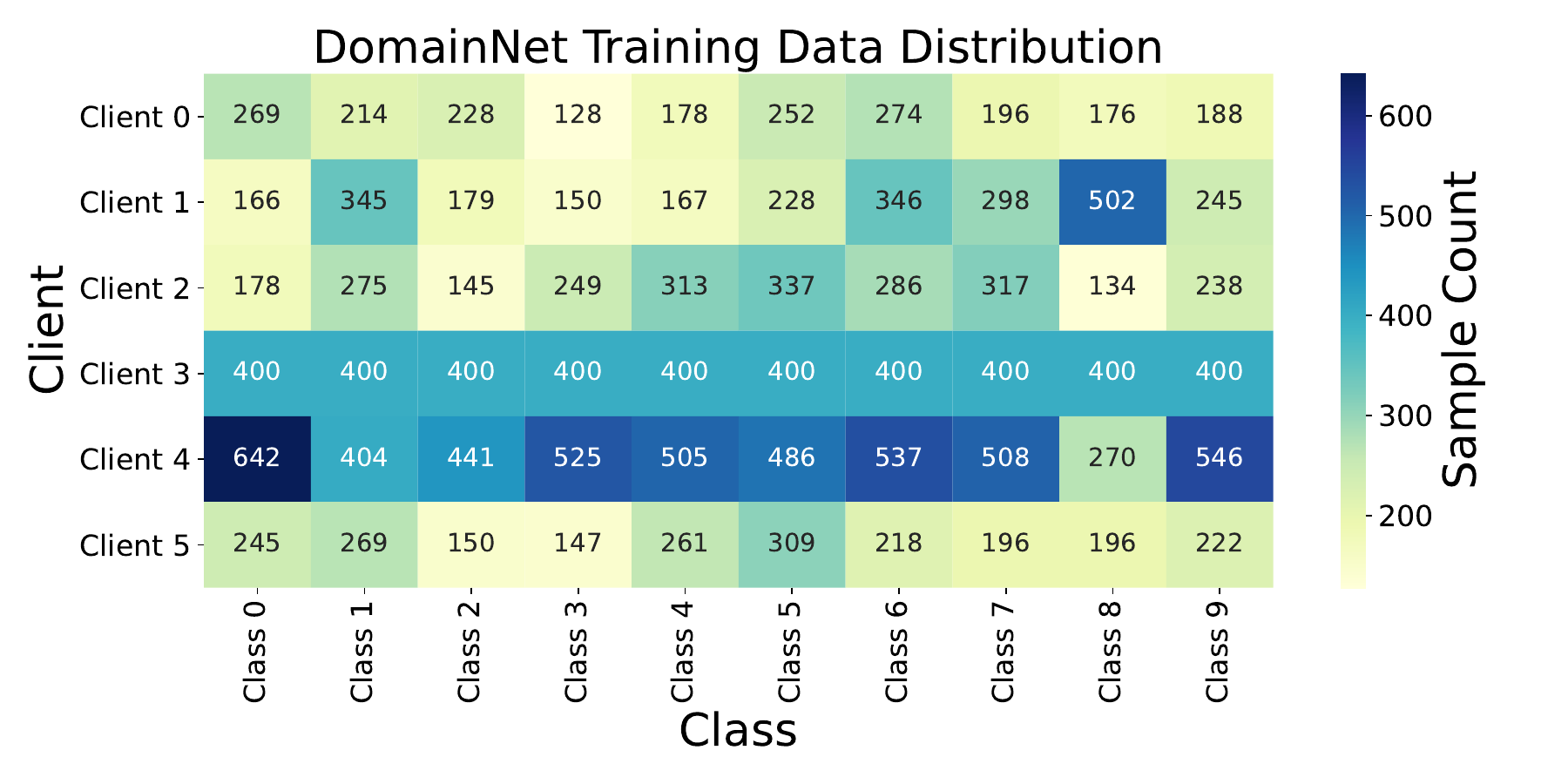}
    \caption{DomainNet clients data distribution.}
    \label{fig:DomainNet_feature}
\end{minipage}\hfill
\begin{minipage}[H]{0.48\textwidth}
    \centering
    \includegraphics[width=\linewidth]{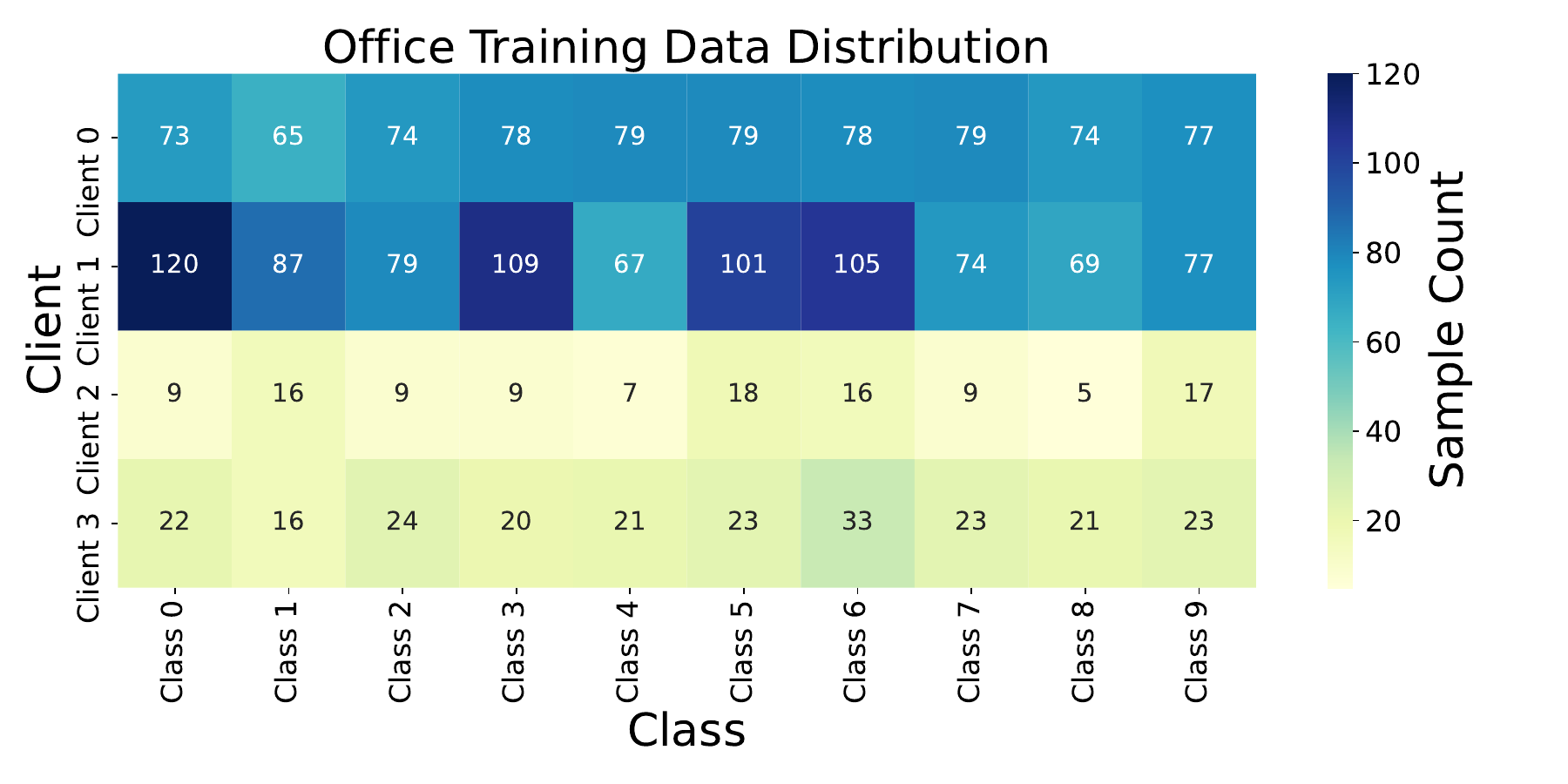}
    \caption{Office-10 clients data distribution.}
    \label{fig:Office_feature}
\end{minipage}
\end{figure}

\paragraph{Hybrid Distribution}
We simulate the hybrid data distribution by combining both label distribution skew and feature distribution shift. We use the same two datasets as in feature shift experiments: Office-10 and DomainNet. To introduce label skew, for each domain, we randomly sample $5$ clients and assign to each client only $2$ out of $10$ total classes. This results in $20$ clients for the Office-Caltech10 dataset ($4$ domains $\times$ $5$ clients) and $30$ clients for DomainNet ($6$ domains $\times$ $5$ clients). This creates a hybrid non-IID setting where clients differ significantly in both input distribution and output distribution. We use the same preprocessing and training configurations as the feature shift experiments. All input images are resized to $256\times 256\times3$ before being fed into $AlexNet$. Models are trained using cross-entropy loss and Adam optimizer with learning rate of $10^{-2}$. The batch size is set to $32$ for Office-10 and $64$ for DomainNet. For DomainNet, we selected the $10$ most frequent as feature shift experiments.To effectively visualize the distribution of local training data across $30$ clients, we used a dot matrix plot, which offers a compact and intuitive representation of client-level variation. The visualization of the Clients distribution of DomainNet and Office-10 datasets are shown in Fig.\ref{fig:DomainNet_mix} and Fig.\ref{fig:Office_mix}

\begin{figure}[H]
\begin{minipage}[H]{0.48\textwidth}
    \centering
    \includegraphics[width=\linewidth]{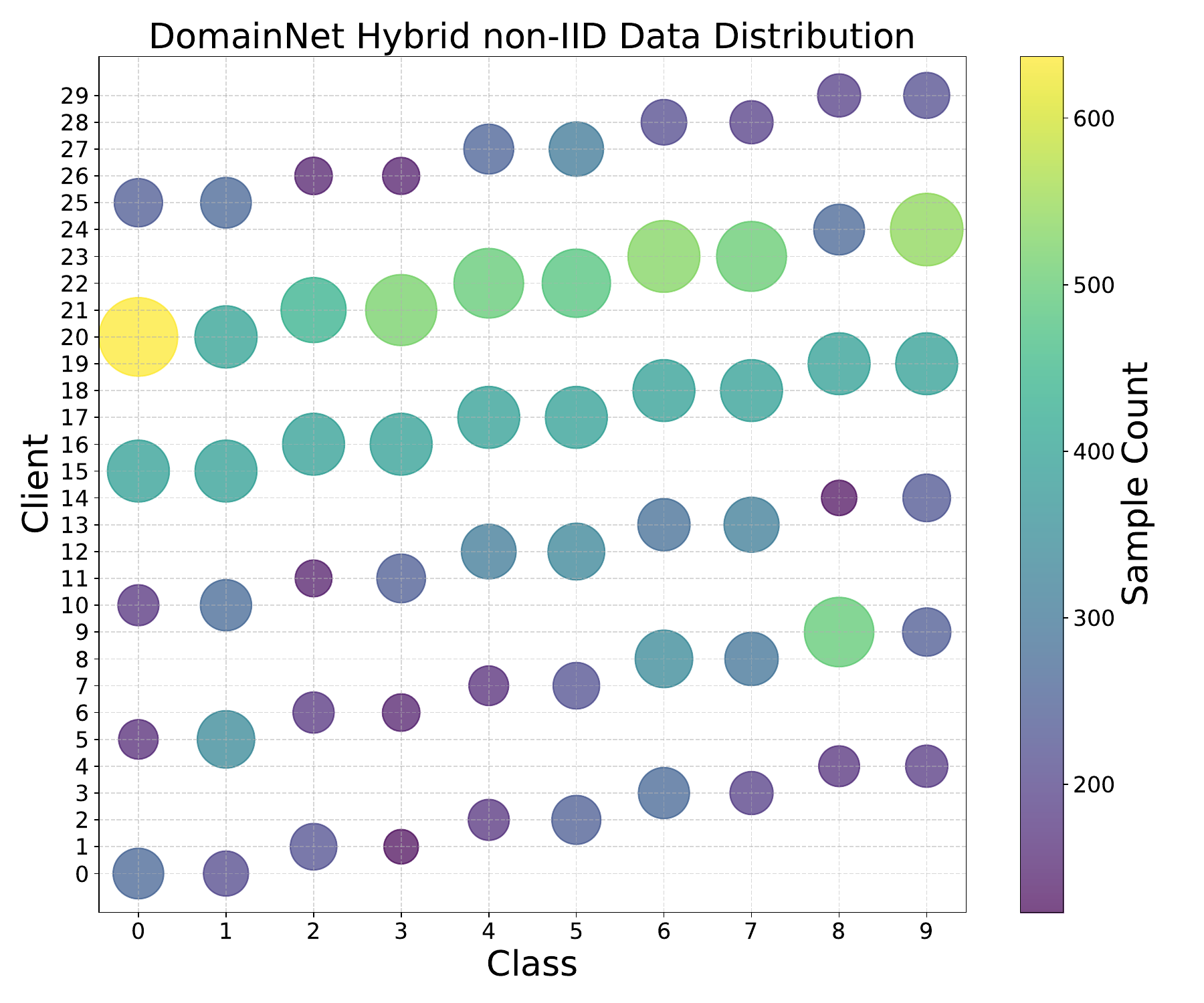}
    \caption{DomainNet clients Hybrid data distribution.}
    \label{fig:DomainNet_mix}
\end{minipage}\hfill
\begin{minipage}[H]{0.48\textwidth}
    \centering
    \includegraphics[width=\linewidth]{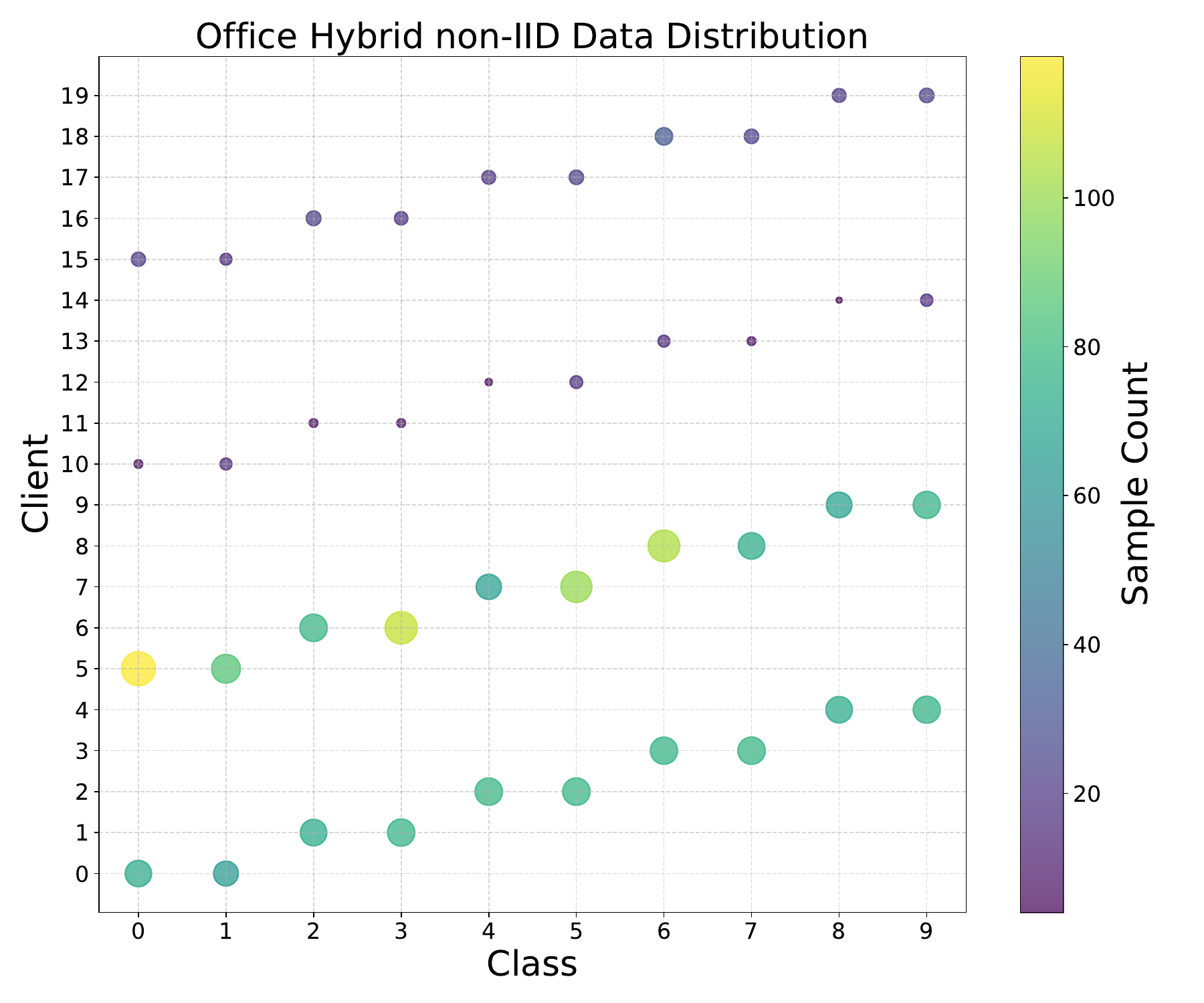}
    \caption{Office-10 clients Hybrid data distribution.}
    \label{fig:Office_mix}
\end{minipage}
\end{figure}

\section{Practical Impact of \pfedct}
\label{app:practical-impact}

\pfedct addresses data heterogeneity in personalized federated learning (PFL) via a fine-grained collaboration mechanism that lets each client selectively rely on collective expertise, aiming to improve accuracy and robustness. This is particularly relevant in domains with substantial variability (e.g., healthcare, finance, recommendation), where traditional federated methods can struggle. Empirically, \pfedct often outperforms strong PFL baselines and, in our evaluated settings, local and centralized training across label skew, feature shift, and hybrid heterogeneity; where margins are small, it performs comparably. Its design limits disclosure by sharing only hard predictions on a shared unlabeled dataset, reducing potential privacy leakage relative to parameter sharing. This follows “share as little as possible”~\citep{mian2023nothing, tan2022fedproto} and aligns with privacy-by-design~\citep{cavoukian2009privacy}. In addition, our differentially private variant (DP-\pfedct) illustrates how to obtain formal $(\varepsilon,\delta)$-DP guarantees for the released signals (labels and expertise), with the privacy accounting provided and empirical calibration left to future work. Finally, federated co-training is communication-efficient for large models: when parameter counts vastly exceed $|U|$, sending hard labels (and one expertise scalar per example) can reduce uplink by orders of magnitude. Combining this with communication-efficient protocols~\citep{kamp2016communication, kamp2019black} has the potential to reduce communication by several orders of magnitude, in particular for large transformer-based models, such as LLMs.

\section{Notation}

\centerline{\bf Federated Learning Setup}
\bgroup
\def\arraystretch{1.5}
\begin{tabular}{p{1.25in}p{3.25in}}
$\displaystyle m$ & Number of participating clients \\
$\displaystyle i \in [m]$ & Index of a client \\
$\displaystyle D_i$ & Private dataset of client $i$ \\
$\displaystyle U$ & Shared public unlabeled dataset used for co-training \\
$\displaystyle T$ & Total number of communication rounds \\
$\displaystyle b$ & Communication period (local steps between rounds) \\
$\displaystyle A_i$ & Local learning algorithm used by client $i$ \\
\end{tabular}
\egroup
\vspace{0.25cm}

\centerline{\bf Models and Predictions}
\bgroup
\def\arraystretch{1.5}
\begin{tabular}{p{1.25in}p{3.25in}}
$\displaystyle h_i^t$ & Local model of client $i$ at round $t$ \\
$\displaystyle L(h, D)$ & Loss of model $h$ on dataset $D$ \\
$\displaystyle \ell_{\text{priv}} = L(h_i^{t-1}, D_i)$ & Private loss on client $i$'s local data \\
$\displaystyle \ell_{\text{pseudo}} = L(h_i^{t-1}, P^t)$ & Loss on pseudo-labeled public data $P^t$ \\
$\displaystyle L_i^t \in \{0,1\}^{|U|\times C}$ & One-hot prediction matrix from client $i$ on public data \\
$\displaystyle E_i^t \in (0,\infty)^{|U|}$ & Confidence (expertise) vector from client $i$ on public data \\
$\displaystyle S^t = \sum_{i=1}^m \text{diag}(E_i^t) \cdot L_i^t$ & Weighted score matrix used for consensus aggregation \\
$\displaystyle L^t[j] = \arg\max_{c\in[C]} S^t[j,c]$ & Consensus pseudo-label for public example $x_j \in U$ \\
\end{tabular}
%\egroup
\vspace{0.25cm}

\centerline{\bf Adaptive Weighting Mechanism}
\bgroup
\def\arraystretch{1.5}
\begin{tabular}{p{1.25in}p{3.25in}}
$\displaystyle \lambda_i^t$ & Adaptive weight controlling trust in global signal for client $i$ at round $t$ \\
$\displaystyle \ell = \ell_{\text{priv}} + \lambda_i^t \cdot \ell_{\text{pseudo}}$ & Total loss used for local model update at round $t$ \\
\end{tabular}
\egroup
\vspace{0.25cm}

\centerline{\bf Optimization and Convergence}
\bgroup
\def\arraystretch{1.5}
\begin{tabular}{p{1.25in}p{3.25in}}
$\displaystyle \theta$ & Model parameters \\
$\displaystyle \nabla L(\theta)$ & Gradient of loss with respect to model parameters \\
$\displaystyle \sigma^2$ & Bounded variance of local gradient estimator \\
$\displaystyle \tilde{\sigma}^2$ & Bounded variance of global gradient estimator (pseudo-label noise) \\
$\displaystyle \delta$ & Bounded drift in local objectives across rounds \\
$\displaystyle L$ & Smoothness constant (Lipschitz constant of the gradient) \\
\end{tabular}
\egroup
\vspace{0.25cm}

\centerline{\bf Sets and Indexing}
\bgroup
\def\arraystretch{1.5}
\begin{tabular}{p{1.25in}p{3.25in}}
$\displaystyle [m] = \{1,\dots,m\}$ & Index set of all clients \\
$\displaystyle [C] = \{1,\dots,C\}$ & Index set of all classes \\
$\displaystyle x_j \in U$ & $j$-th public unlabeled sample \\
$\displaystyle y_j$ & True (unknown) label of public sample $x_j$ \\
$\displaystyle |U|$ & Number of samples in the public dataset $U$ \\
$\displaystyle |D_i|$ & Number of samples in the local dataset of client $i$ \\
$\displaystyle L_i^t[j,c]$ & $(j,c)$-th entry of prediction matrix $L_i^t$ \\
$\displaystyle E_i^t[j]$ & Confidence of client $i$ on public example $x_j$ \\
\end{tabular}
\egroup

\end{document}